\documentclass[manuscript,screen]{acmart}

\AtBeginDocument{%
  }

\setcopyright{acmlicensed}
\copyrightyear{2018}
\acmYear{2018}
\acmDOI{XXXXXXX.XXXXXXX}

\acmConference[Conference acronym 'XX]{Make sure to enter the correct
  conference title from your rights confirmation emai}{June 03--05,
  2018}{Woodstock, NY}
\acmISBN{978-1-4503-XXXX-X/18/06}

\usepackage{microtype}
\usepackage{graphicx}
\usepackage{subfigure}
\usepackage{booktabs} 
\usepackage{bm}
\usepackage{hyperref}

\usepackage{amsmath}
\usepackage{mathtools}
\usepackage{amsthm}
\usepackage[capitalize,noabbrev]{cleveref}
\theoremstyle{plain}
\newtheorem{theorem}{Theorem}[section]

\theoremstyle{definition}
\newtheorem{definition}[theorem]{Definition}

\theoremstyle{remark}

\newtheorem{claim}[theorem]{Claim}
\newtheorem{observation}[theorem]{Observation}

\usepackage{xcolor}         
\usepackage{amsmath,amsfonts,amsthm}
\usepackage{graphicx,xcolor}
\usepackage{wrapfig}
\usepackage{caption}
\usepackage{subcaption}
\usepackage{multirow}
\usepackage{array}
\newcolumntype{P}[1]{>{\centering\arraybackslash}p{#1}}
\usepackage[framemethod=tikz]{mdframed}
\usepackage{cancel}
\DeclareMathOperator*{\argmax}{argmax}
\DeclareMathOperator*{\argmin}{argmin}
\usepackage{wrapfig}
\usepackage{appendix}
\usepackage{lipsum}

\usepackage[textsize=tiny]{todonotes}

\begin{document}

\title{Balancing Fairness and Accuracy in Data-Restricted Binary Classification}

\author{Zachary McBride Lazri}
\email{zlazri@umd.edu}
\orcid{0000-0002-7400-5547}
\affiliation{%
  \institution{University of Maryland}
  \city{College Park}
  \state{Maryland}
  \country{USA}
}

\author{Danial Dervovic}
\email{danial.dervovic@jpmchase.com}
\affiliation{%
  \institution{J.P. Morgan AI Research}
  \city{London}
  \country{UK}
}

\author{Antigoni Polychroniadou}
\email{antigoni.polychroniadou@jpmorgan.com}
\affiliation{%
  \institution{J.P. Morgan AI Research}
  \city{New York}
  \state{New York}
  \country{USA}
}

\author{Ivan Brugere}
\email{ivan.brugere@jpmchase.com}
\affiliation{%
  \institution{J.P. Morgan AI Research}
  \city{New York}
  \state{New York}
  \country{USA}
}

\author{Dana Dachman-Soled}
\email{danadach@umd.edu}
\affiliation{%
  \institution{University of Maryland}
  \city{College Park}
  \state{Maryland}
  \country{USA}
}

\author{Min Wu}
\email{minwu@umd.edu}
\affiliation{%
  \institution{University of Maryland}
  \city{College Park}
  \state{Maryland}
  \country{USA}
}

\renewcommand{\shortauthors}{Trovato et al.}

\begin{abstract}
  Applications that deal with sensitive information may have restrictions placed on the data available to a machine learning (ML) classifier. For example, in some applications, a classifier may not have direct access to sensitive attributes, affecting its ability to produce accurate and fair decisions. This paper proposes a framework that models the trade-off between accuracy and fairness under four practical scenarios that dictate the type of data available for analysis. Prior works examine this trade-off by analyzing the outputs of a scoring function that has been trained to implicitly learn the underlying distribution of the feature vector, class label, and sensitive attribute of a dataset. In contrast, our framework directly analyzes the behavior of the optimal Bayesian classifier on this underlying distribution by constructing a discrete approximation it from the dataset itself. This approach enables us to formulate multiple convex optimization problems, which allow us to answer the question: \textit{How is the accuracy of a Bayesian classifier affected in different data restricting scenarios when constrained to be fair?} Analysis is performed on a set of fairness definitions that include group and individual fairness. Experiments on three datasets demonstrate the utility of the proposed framework as a tool for quantifying the trade-offs among different fairness notions and their distributional dependencies. 
\end{abstract}



\received{20 February 2007}
\received[revised]{12 March 2009}
\received[accepted]{5 June 2009}

\maketitle

\section{Introduction}
A variety of studies have found bias to exist in machine learning (ML) models~\cite{sweeney2013discrimination, angwin2016machine,LarsonCompas, buolamwini2018gender, larson2017we}, raising concerns over their use in high-stakes applications. For example, Angwin et al.~\cite{angwin2016machine} and Larson et al.~\cite{LarsonCompas} found that a tool used to calculate the risk of criminal defendants repeating a crime was biased against African Americans. To address these concerns, a variety of mathematical definitions have been constructed to quantify the fairness of such models, the most prominent of which fall under two categories—group fairness~\cite{kamiran2012data, hardt2016equality, chouldechova2017fair, zafar2017fairness,calib_fair} and individual fairness~\cite{dwork2012fairness, petersen2021post}. Group fairness definitions aim to quantify biases that may exist among the results produced for two or more demographic groups, while individual fairness focuses on ensuring the fair treatment of similar individuals. However, multiple theorems have shown the impossibility of satisfying multiple fairness definitions simultaneously. This raises the following questions: 
\begin{itemize}
    \item To what extent is it possible to simultaneously satisfy multiple definitions of fairness exactly or in some relaxed form?
    \item If it is possible to simultaneously satisfying multiple fairness definitions, how much accuracy is achievable?
\end{itemize}
Noting that the Bayesian classifier produces the smallest probability of error among all classifiers, the answer to the latter question requires determining a Bayesian classifier’s accuracy when it is required to satisfy fairness constraints. Answering the former question requires verifying that a classifier which satisfies such constraints exists. 

We should also be aware that the answers to these questions depend on the underlying distribution of features available to a classifier. In the simplest case, all features are made available. However, in many applications in which one or more features are deemed sensitive, which we refer to as sensitive attributes, restrictions are set up to limit the features available to a classifier. For example, unless certain exceptions apply, financial institutions are not allowed to request or consider an applicant's race, among other prohibited bases, when applying for loans, but must prove that their decisions are anti-discriminatory under the Fair Housing Act and Equal Credit Opportunity Act \cite{FHA,FECOA}. Similarly, the Civil Rights Act of 1964 \cite{berg1964equal} requires that public schools and institutions of higher education do not discriminate on the basis of several sensitive attributes, including race, sex, and religion. However, analyzing the equity of their decisions is not always directly possible since it is not mandatory for applicants to provide such demographic information in their applications. Such situations are captured by the definition of unawareness of a sensitive attribute proposed by Kusner et al.~\cite{kusner2017counterfactual}. In certain situations, a model is only permitted to use features that have been decorrelated with respect to the sensitive attribute to make decisions. Financial institutions, for example, commonly form separate ML and compliance teams in the same institution. Compliance teams oversee the handling of sensitive information and are responsible for ensuring that the decisions produced by a company are non-discriminatory, and ML teams train models to produce decisions for a company. While a compliance team is provided with access to all sensitive attributes, when available, an ML team is prohibited from accessing such features and should not be able to deduce them from the data~\cite{de2020cryptocredit}. In other words, the features provided to the ML team must be decorrelated with respect to the sensitive attribute, a notion introduced by Zemel et al.~\cite{zemel2013learning}.

Motivated by these circumstances, this paper analyzes the trade-off in accuracy that a Bayesian classifier must incur when it is required to satisfy multiple fairness constraints under different situations that restrict the feature distribution available to it. Obtaining the fair Bayesian classifier requires accessing the true distribution of a given set of features. Yet in practice, we do not have access to the true feature distribution; rather we are provided with a sampling from it in the form of a dataset. Thus, our objectives in this work are twofold: (1) given a dataset, we aim to faithful approximate its underlying feature distribution for analysis, and (2) given the approximated feature distribution, we aim to analyze the decisions produced by a Bayesian classifier when it is constrained to be fair. 

Unfortunately determining a closed form expression for the underlying feature distribution of a real-world dataset is impractical since it may be extremely complex. Thus, we instead propose to directly approximate a feature distribution as a discrete signal. Specifically, we use vector quantization (VQ)~\cite{gersho1991vector} to induce a cell decomposition over the feature space, with each cell representing an element on the support of our discrete approximation. The distribution over these VQ cells is obtained by accumulating the statistics of the samples that fall inside each cell. Since the number of samples in a VQ cell may be small, we densely sample a generator which has been trained to learn the underlying feature distribution of a dataset to ensure that the statistics accumulated over each cell faithfully approximate the true underlying feature distribution. This process is outlined in Fig. \ref{framework_fig}a.

In our subsequent analysis, we derive an optimization formulation to model the decisions produced by the Bayesian classifier of this discrete distribution when it is constrained to satisfy various fairness definitions. The Bayesian classifier is designed to be stochastic and thus assigns scores to different feature vectors, reflecting their probability of receiving a positive class label.  All analyzed fairness definitions can be directly encoded as convex constraints in our framework, allowing us to avoid proactively constructing objective functions that indirectly satisfy them. While commonly analyzed in isolation, we concurrently formulate constraints from individual and group fairness definitions to investigate their relationship. Our analysis is performed under four 
data-restricting situations listed in Fig. \ref{framework_fig}b., under which the sensitive attribute \textbf{is (is not)} available and the features used for classification \textbf{are (are not)} required to be decorrelated from it. We elaborate on these situations in Section \ref{prob_setup}.
Experiments conducted on three public datasets reveal that our framework captures the distributional dependence of the tensions that exist between different fairness notions and suggest that coupling individual and group fairness prevents a Bayesian classifier from arbitrarily penalizing individuals to satisfy group fairness. We also observe that a fair Bayesian classifier is typically able to maintain its accuracy on feature vectors that have been decorrelated from the sensitive attribute.

In summary, the contributions of this work are as follows:
\begin{itemize}
    \item We introduce a framework that can be used as a tool to analyze the underlying feature distribution of a dataset and determine the compatibility of different fairness notions in performing classification.
    \item We introduce a method for constructing a high quality discrete approximation of the underlying feature distribution of a dataset.
    \item We construct multiple convex optimization problems to analyze the trade-off between accuracy and individual and group fairness notions in various practical data restricting scenarios.
    \item We present experimental results to analyze our framework on three benchmark datasets.
\end{itemize}

The remainder of this paper is organized as follows. 
In Section \ref{related_work} we review the related work in analyzing the trade-off between accuracy and fairness. Section \ref{background} provides the definitions of fairness that we analyze in our trade-off analysis along with descriptions of the restrictions placed on the data distribution under each analysis scenario. In Section \ref{prob_setup}, we formulate a framework for analyzing the reduction in accuracy incurred from forcing a Bayesian classifier to be fair under four scenarios that constrain the information available to a Bayesian classifier. Experimental results are presented in Section \ref{exp_results} to quantify the trade-off between accuracy and fairness under each of these scenarios. Finally, we conclude the paper and provide discussion in Section \ref{conclusion}.

\begin{figure*}[t]
\centerline{\includegraphics[width=0.95\textwidth]{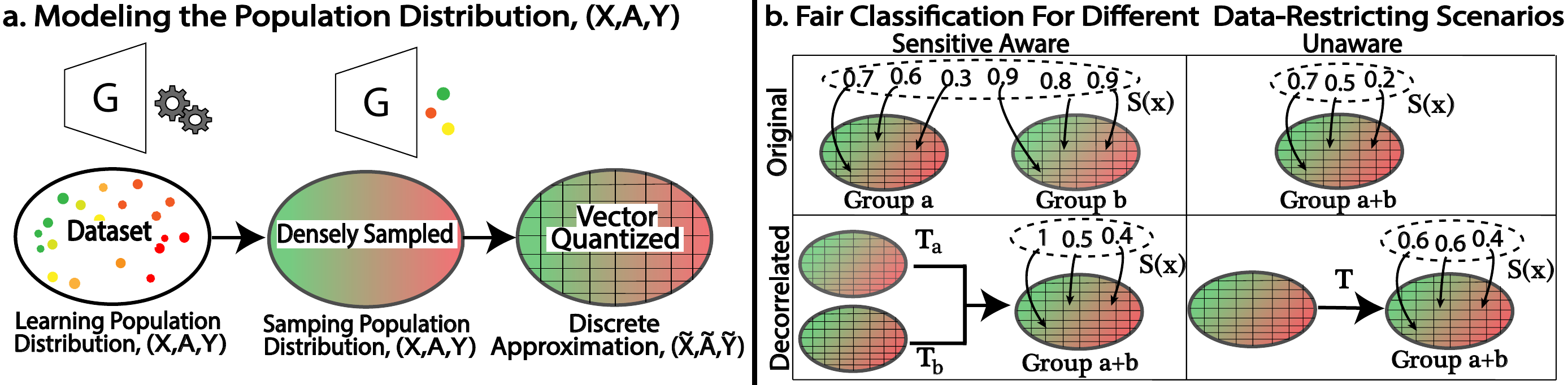}}
\caption{Outline of proposed framework for trade-off analysis. (a) A discrete approximation of the population distribution is constructed by using a generator to densely sample it and applying vector quantization. (b) The accuracy-fairness trade-off is analyzed under four data-restricting situations in which the sensitive attribute \textbf{is (is not)} available and the features used for classification \textbf{are (are not)} required to be decorrelated from it.}
\label{framework_fig}
\end{figure*}

\section{Related Work} 
\label{related_work}
Multiple works have shown the impossibility of simultaneously satisfying multiple definitions of fairness \cite{chouldechova2017fair, kleinberg2016inherent, zhao2022inherent}. 
For example, Chouldechova~\cite{chouldechova2017fair} showed that attempting to exactly satisfy three group fairness definitions simultaneously is futile, while Kleinberg et al.~\cite{kleinberg2016inherent} showed that demographic parity and equalized odds are in conflict when demographic groups have unequal class label balance. This has motivated a number of studies to analyze the trade-off between the accuracy and fairness of ML outcomes, the majority of which analyze a single fairness notion using pre-processing \cite{calmon2017optimized, kamiran2012data, luong2011k}, in-processing \cite{chen2020towards, jiang2020wasserstein, zafar2017fairness}, or post-processing ~\cite{petersen2021post, lohia2019bias} methods. Other studies have been conducted to analyze the trade-off between accuracy and multiple fairness definitions \cite{kim2020fact, liu2022accuracy,Hsu, Celis}. Kim et al.~\cite{kim2020fact} analyzed the trade-off between accuracy and fairness by constructing primarily linear constraints from the joint distribution between the class label and sensitive attribute of a dataset,
though they do not consider the correlations that the class label and sensitive attribute may share with a dataset's feature vectors in analyzing this trade-off. Liu et al.~\cite{liu2022accuracy} propose an in-processing method for analyzing the trade-off between accuracy and multiple fairness definitions, though they are required to use proxy fairness constraints to avoid solving a non-convex optimization problem. 
Hsu et al.~\cite{Hsu} analyze the trade-off between accuracy and multiple fairness definitions through post-processing the scores produced by a model. However, such an approach may provide less flexible analysis since the transformation from the space of feature vectors to the space of scores is many-to-one, which could force many non-similar feature vectors to be treated similarly. Overall, none of these works incorporate both individual and group fairness notions into their analysis, nor do they directly analyze how the restrictions on the information available to a model may affect their analyses. Like our work, Menon et al. ~\cite{menon2018cost} also study the trade-off between accuracy and fairness for a Bayesian classifier. However, their study is purely theoretic and focuses only on two group fairness notions. Conversely, our framework can be used as a practical tool to uncover the tensions among a variety of different fairness definitions that exist in real datasets.

\section{Preliminaries for Trade-off Analysis}
\label{background}
The following setup is used to explicitly formulate each definition analyzed in our framework. Let $X, A, \text{ and } Y$ respectively represent the random feature vector, sensitive attribute, and class label associated with the sample spaces $\mathcal{X}, \mathcal{A}, \text{ and } \mathcal{Y}$. For simplicity, assume that $\mathcal{X}=\mathbb{R}^k, \text{ } \mathcal{A}=\{a,b\}, \text{ and } \mathcal{Y}=\{0,1\}$, though our formulations may be generalized to non-binary sensitive attributes through combinatorial extension. Let $S:\mathcal{X}\rightarrow[0,1]$ be a scoring function whose output represents the conditional probability with which we assign a feature vector a label of 1, which we refer to as a score. Then, its associated randomized estimator is:
\begin{equation}
    \hat{Y}(\mathbf{x})=\begin{cases}
        1&\text{, w.p. } S(\mathbf{x})\\
        0&\text{, w.p. } 1-S(\mathbf{x})
    \end{cases}\notag.
\end{equation}
\bigskip

\noindent
\underline{\textbf{Fairness Definitions}}

Different fairness definitions proposed in the literature provide tools for ensuring that ML models uphold various societal values. 
While Demographic Parity \cite{kamiran2012data} aims to ensure class label balance across demographic groups, Equalized Odds \cite{hardt2016equality} focuses on ensuring that ML models are equally effective in classifying different demographic groups. On the other hand, definitions of individual fairness~\cite{dwork2012fairness,petersen2021post} focus on ensuring that individuals are neither marginalized nor preferentially treated. 
In our framework, we analyze the following set of fairness definitions: 

\begin{definition}
    Demographic Parity \cite{kamiran2012data}. An estimator, $\hat{Y}$, satisfies demographic parity for a binary sensitive attribute, $A\in\{a,b\}$, if it maintains equal positive prediction rates for each value of the sensitive attribute:
    \begin{equation}
        P(\hat{Y}(X)=1|A=a) = P(\hat{Y}(X)=1|A=b).\notag
    \end{equation}
\end{definition}
\begin{definition}
    Equal Accuracy \cite{zafar2017fairness}. An estimator, $\hat{Y}$, satisfies equal accuracy for a binary sensitive attribute, $A\in\{a,b\}$, if its accuracy for different values of the sensitive attribute is equal:
    \begin{equation}
        P(\hat{Y}(X)=Y|A=a) = P(\hat{Y}(X)=Y|A=b).\notag
    \end{equation}
\end{definition}
\begin{definition}
    Equal Opportunity \cite{hardt2016equality}. An estimator, $\hat{Y}$, satisfies equal opportunity for a binary feature, $A\in\{a,b\}$, if the true positive rate for each value of the sensitive attribute is equal:
    \begin{equation}
        P(\hat{Y}(X)=1|A=a,Y=1) = P(\hat{Y}(X)=1|A=b,Y=1). \notag
    \end{equation}
\end{definition}
\begin{definition}
    Predictive Equality \cite{corbett2017algorithmic}. An estimator, $\hat{Y}$, satisfies predictive equality for a binary feature, $A\in\{a,b\}$, if the true negative rate for each value of the sensitive attribute is equal:
    \begin{equation}
        P(\hat{Y}(X)=1|A=a,Y=0) = P(\hat{Y}(X)=1|A=b,Y=0). \notag
    \end{equation}
\end{definition}
\begin{definition}
    Equalized Odds \cite{hardt2016equality}. An estimator, $\hat{Y}$, satisfies equalized odds if both equal opportunity and predictive equality are satisfied.
\end{definition}
\begin{definition}
    Local Individual Fairness \cite{petersen2021post}. A scoring function, $S$, is locally individually fair if individuals with similar features according to a distance measure, $d_{\mathcal{X}}$, receive similar scores according to a distance measure, $d_{\mathcal{Y}}$. That is, for $L\geq0$,\\
    \begin{equation}
        \mathbb{E}_{\mathbf{x}_i\sim P_X}\left[\limsup_{\mathbf{x}_i:d_{\mathcal{X}}(\mathbf{x}_i,\mathbf{x}_j)\downarrow 0}\frac{d_{\mathcal{Y}}(S(\mathbf{x}_i),S(\mathbf{x}_j))}{d_{\mathcal{X}}(\mathbf{x}_i,\mathbf{x}_j)}\right]\leq L\leq\infty.\notag
    \end{equation}
\end{definition}
Local Individual Fairness is similar to the standard notion of Individual Fairness introduced by Dwork et al.\cite{dwork2012fairness} with the main difference being that it only places similarity constraints on pairs of \textit{similar} individuals, not \textit{all} pairs of individuals. Thus, Local Individual Fairness still preserves the philosophy of treating similar individuals similarly, but it also requires satisfying fewer constraints than the standard Individual Fairness definition. This latter point is critical since the standard definition of Individual Fairness produces $\mathcal{O}(N^2)$ constraints, which leads to the poor scalability of optimization problems that aim to satisfy it. 
Specifically, Local Individual Fairness allows us to construct the following set of constraints:
\begin{align}
    &e^{-\theta d_{\mathcal{X}}^2(\mathbf{x}_i,\mathbf{x}_j)}|s^F[i]-s^F[j]|\leq \epsilon_{IF}, \ \ \ \ \ \ \forall i,j \ \ \text{  s.t. } \ \  d_{\mathcal{X}}(\mathbf{x}_i,\mathbf{x}_j)\leq \eta,\notag
\end{align}
Here, only feature vectors falling within an $\eta$-neighborhood of each other are required to receive similar scores. For small to moderate sizes of $\eta$ this set of constraints is much smaller than the set produced by the standard Individual Fairness definition. \bigskip

\noindent
\underline{\textbf{Data-Restricting Definitions}}

A model's ability to balance accuracy and fairness depends on the data available to it. We provide two practical definitions that dictate the data available to a model, which we use to create the four scenarios under which we perform our analyses.
\begin{definition}
    Unawareness of Sensitive Attribute \cite{kusner2017counterfactual}. The sensitive attribute is not a feature in the space of feature vectors.
\end{definition}
\begin{definition}
    \label{decorrelation_def}
    Decorrelation with the Sensitive Attribute \cite{zemel2013learning}. The space of feature vectors is decorrelated with the sensitive attribute. That is, the following property is satisfied: 
    \begin{equation}
        P(X|A=a) = P(X|A=b). \notag
    \end{equation}
\end{definition}

\section{Framework for Analysis}
\label{prob_setup}

In this section, we present the framework used to analyze the trade-off between accuracy and fairness under four scenarios. An illustration of the first stage of our framework is provided in Fig.\ref{framework_fig}a. 
We construct a discrete approximation, $(\tilde{X},\tilde{A},\tilde{Y})$, of $(X,A,Y)$ by first using a generator, $G$, that has learned to produce samples from the distribution, $(X,A,Y)$. We then partition the feature space into $N$ non-intersecting cells, $\{C_i\}_{i=1}^{N}$, that cover $\mathcal{X}$ using vector quantization (VQ) 
~\cite{linde1980algorithm,lloyd1982least}. The support of $\tilde{X}$ is given by  $\{\mathbf{x}_i |\mathbf{x}_i\in \mathcal{X}\}_{i=1}^{N}$, where $\mathbf{x}_i$ represent the centroid of $C_i$. We infer the population statistics of a given cell from the samples inside of it. By densely sampling $G$, we ensure that $(X\in C_j,A,Y) \approx (\tilde{X}=\mathbf{x}_j,\tilde{A},\tilde{Y})$ (see Appendix \ref{Model} for details). For notational simplicity, we use $(X,A,Y)$ in place of $(\tilde{X}, \tilde{A},\tilde{Y})$ in the remainder of this paper.

An illustration of the different scenarios under which we analyze the fairness-accuracy trade-off is provided in Fig.\ref{framework_fig}b. The first situation is the unconstrained situation in which we may make direct use of the sensitive attribute to separately assign scores to the feature vectors of different groups (see Appendix \ref{sec_fair_aware} for this formulation). The second situation is formulated in Section \ref{fairness}, under which the sensitive attribute is unavailable, meaning that a Bayesian classifier must assign the same score to individuals from different groups with the same feature vectors. When the sensitive attribute is required to be decorrelated with the feature vectors used for classification, an added layer of processing is required to decorrelate the feature vectors from the sensitive attribute prior to providing them to a Bayesian classifier. In the third situation, access to the sensitive attribute allows us to construct two separate mappings 
that redistribute the feature vectors associated with each group to achieve this goal (see Appendix \ref{sec_priv_aware} for this formulation). In the fourth situation, the sensitive attribute is unavailable, meaning that a single mapping
must be applied to the features of both groups to achieve this goal. This situation is formulated in Section \ref{fair_priv}.


\paragraph{Consolidating Notation.}
Here, we introduce matrix--vector notation which is used in the ensuing sections for modeling purposes. 
Bold face capital letters represent matrices, e.g. $\mathbf{X}$, where the value of the $i^{th}$ row and $j^{th}$ column is given by $\mathbf{X}[i,j]$. Bold face lower case letters represent column vectors, e.g. $\mathbf{x}$, where the $i^{th}$ element is given by $\mathbf{x}[i]$. The transpose of a matrix or vector is denoted with superscript $T$, e.g. $\mathbf{X}^T$. $\mathbf{p}$ is used to capture a joint distribution with $X$ and other variables. Since $\mathcal{A}=\{a,b\}$, superscripts (subscripts) on $\mathbf{p}$ containing letters refer to joint (conditional) distributions with the sensitive attribute. Similarly, since $\mathcal{Y}=\{0,1\}$, superscripts (subscripts) 
 on $\mathbf{p}$ containing numbers refer to joint (conditional) distributions with the class label.  
The following examples illustrate this notation. The $i^{th}$ elements of the vectors $\mathbf{p}_{a}, \text{ } \mathbf{p}_{0}, \text{ and } \mathbf{p}_0^a$ are equal to $P(X=\mathbf{x}_i|A=a), \text{ } P(X=\mathbf{x}_i|Y=0), \text{ and } P(X=\mathbf{x}_i,A=a|Y=0)$, respectively. We use $\mathbf{1}_{N}$ and $\mathbf{0}_{N}$ to represent column vectors of length $N$, containing all $1$s and all $0$s, respectively. $\mathbf{I}_{N,N}$ represents an $N\times N$ identity matrix, while $\mathbf{0}_{M,N}$ and $\mathbf{1}_{M,N}$ represent  $M\times N$ matrices of all zeros and ones, respectively. Finally, we use the superscripts $*$ and $(F)$ to distinguish between variables associated with the unconstrained and fair Bayesian classifiers, respectively. A table containing all notation introduced in this paper is provided in Appendix~\ref{notation_summary}.

\subsection{Fairness-Accuracy Trade-off}
\label{fairness}
In this section, we derive an optimization problem that obtains the scoring function of the fair Bayesian classifier. Since we analyze discrete distributions, this scoring function is characterized by the scoring vector, $\mathbf{s}^{(F)}\in [0,1]^{N}$, in which each element provides the probability that the fair Bayesian classifier assigns a class label of 1 to a feature vector sampled from a given VQ cell. We solve for $\mathbf{s}^{(F)}$ by first finding the unconstrained Bayesian classifier's scoring vector, $\mathbf{s}^{*}\in [0,1]^{N}$, and determining how to optimally deviate the fair Bayesian classifiers scores from it so that we minimize its accuracy drop off while satisfying the necessary fairness constraints.

Given access to the joint distribution $(X,Y)$, a Bayesian classifier produces the most accurate solution by taking the majority vote over the support of $X$. This means its associated scoring vector is binary and is given by the following solution:
\begin{equation}
    \mathbf{s}^{*}[i] = \argmax_{y}\mathbf{p}^{y}[i], \forall i.
\end{equation}
Hence, it is a special case of a randomized classifier that produces deterministic outputs, and its accuracy in terms of the probability of correct prediction is given by $Acc^*=\sum_{i=1}^{N}\mathbf{p}^{\mathbf{s}^{*}[i]}[i]$.


We optimize for $\mathbf{s}^{(F)}$ indirectly by solving for a vector, $\mathbf{m} = \mathbf{s}^{*}-\mathbf{s}^{(F)}$, that represents the deviation between $\mathbf{s}^{(F)}$ and  $\mathbf{s}^{*}$. Given $\mathbf{m}$, the reduction in accuracy incurred by deviating $\mathbf{s}^{(F)}$'s scores from $\mathbf{s}^{*}$'s scores is $(\mathbf{p}^{1}-\mathbf{p}^{0})^T\mathbf{m}$, which we aim to minimize while satisfying the necessary fairness constraints. Since satisfying a particular notion of fairness exactly may be too strict for a variety of applications, we formulate all fairness constraints as inequalities and provide limits on the degree to which the fair classifier may deviate from exactly satisfying a particular notion of fairness. The list of constraints for each group fairness notion is provided below.
\begin{equation}
    \begin{split}
        |(\mathbf{p}_{a}-\mathbf{p}_b)^T(\mathbf{s}^{*}-\mathbf{m})|\leq \epsilon_{DP} \ \ \ \ \ \ \ \ \ (DP)\\
        |(\mathbf{p}_{a,0}-\mathbf{p}_{b,0})^T(\mathbf{s}^{*}-\mathbf{m})|\leq \epsilon_{EOp} \ \ \ (PE) \\
        |(\mathbf{p}_{a,1}-\mathbf{p}_{b,1})^T(\mathbf{s}^{*}-\mathbf{m})|\leq \epsilon_{PE} \ \ \ (EOp)\\
    \end{split}
\quad \quad \quad
    \begin{split}
        \epsilon_{EOp} = \epsilon_{PE} \ \ \ \ \ \ \ \ \ \ \ \ \ \ \ \ \ \ \ \ \ \ \ \ \ \ \ \ \ \ \ \ \ \ \ \ (EOd)\\
        |(\mathbf{p}_{a}^{0}-\mathbf{p}_{b}^{0})^T(\mathbf{1}_{N}-\mathbf{s}^{*}+\mathbf{m})\ \ \ \ \ \ \ \ \ \ \ \ \ \ \ \ \ \ \ \ \ \ \ \\
        +(\mathbf{p}_{a}^{1}-\mathbf{p}_{b}^{1})^T(\mathbf{s}^{*}-\mathbf{m})| \leq \epsilon_{EA} \ \ \ (EA)
    \end{split}
\end{equation} 
The set of local individual fairness constraints 
can be formulated as follows:
\begin{align}
    |\mathbf{W}(\mathbf{s}^{*}-\mathbf{m})|\leq \epsilon_{IF}\mathbf{1}_{B}, \ \ \ (Ind)
\end{align}
where $\mathbf{W}$ is a ${B\times N}$ matrix is a matrix in which the number of rows is equal to the total number of feature vector pairs on the support of $X$ within an $\eta$-neighborhood of each other. In particular, the each row of  $\mathbf{W}$ contains non-zero entries, $e^{-\theta d^2_{\mathcal{X}}(\mathbf{x}_i,\mathbf{x}_j)}$ and $-e^{-\theta d^2_{\mathcal{X}}(\mathbf{x}_i,\mathbf{x}_j)}$ in only two indices, $i$ and $j$, respectively, for which $d_{\mathcal{X}}(\mathbf{x}_i,\mathbf{x}_j)\leq \eta$. 

Letting $\mathcal{I}^+$ and $\mathcal{I}^-$, be the set of indices for which the values of $\mathbf{s}^{*}$ are 1 and 0, respectively, the final optimization problem constructed to analyze the fairness-accuracy trade-off is given as follows:
\begin{align}
    \begin{split}
         \ \ \ \ \ \ \ \ \ \ \ \ \ \ \ \min_{\mathbf{m}}(\mathbf{p}^{1}-\mathbf{p}^{0})^T\mathbf{m},
    \end{split}
    \quad
    \begin{split}
        \text{ s.t. } \ \ \ \ \ \ \ \ \ \
    \end{split}
    \quad
    \begin{split}
        & (3) \text{ and } (4) \text{ are satisfied} \\
        &0\leq \mathbf{m}[i]\leq 1,  \ \ \ \ \ i \in \mathcal{I}^+\\
        -&1\leq \mathbf{m}[i]\leq 0, \ \ \ \ i \in \mathcal{I}^-
    \end{split} \ \ \ \  \ \ \ \ \ \ \ \
\end{align}
Observing that each of the constraints in this minimization problem can be made linear in $\mathbf{m}$, this optimization problem is convex and can be efficiently solved using linear programming \cite{dantzig1963linear}. Thus, $\mathbf{s}^{(F)}=\mathbf{s}^{*}-\mathbf{m}$ and the accuracy of the fair Bayesian classifier is given by $Acc^{(F)}= Acc^*-(\mathbf{p}^{1}-\mathbf{p}^{0})^T\mathbf{m}$. See Appendix \ref{app_LP_const} for the explicit derivation of minimization problem (4).

\subsection{Transfer Fairness to Decorrelated Domain}
\label{fair_priv}
Definition~\ref{decorrelation_def} measures differences in the distribution of the feature vectors for different groups. Unless the original space of feature vectors is decorrelated with respect to the sensitive attribute, a transformation must be applied to $X$ to satisfy this definition. 
Moreover, we require such a transformation to produce feature vectors which the original fair Bayesian classifier  still fairly classifies. 
Since $X$ is discrete in our framework, this transformation comes in the form of a matrix, $\mathbf{T}$, where $\mathbf{T}[i,j]$ represents the probability of mapping a sample from the $j^{th}$ to the $i^{th}$ VQ cell. As a result, we require that $\mathbf{T}\in\mathcal{P}$, where $\mathcal{P}$ represents the set of matrices whose columns are probability mass functions. A constraint on decorrelation can be easily constructed by minimizing the value of $\|\mathbf{T}(\mathbf{p_a}-\mathbf{p_b})\|_1$.
Hence, the following optimization problem is used to achieve our goal.
\begin{align}
    \min_{\mathbf{T}\in\mathcal{P}} 
    &-\lambda\underbrace{(\mathbf{s}^{(F)T}\mathbf{T}\mathbf{p}^{1} + (\mathbf{1}_{N}-\mathbf{s}^{(F)})^T\mathbf{T}\mathbf{p}^{0})}_{Accuracy} + \beta\underbrace{\|\mathbf{T}(\mathbf{p}_a-\mathbf{p}_b)\|_1}_{Correlation}\notag\\
    & \text{s.t. } \ \ \ \ \ \ \ \ \ \ 
    |f(\mathbf{T})|\leq \mathbf{f} \ \ \ \  \textrm{\it Fairness}
\end{align}
The first term represents the negative of the accuracy achieved by applying the fair Bayesian classifier to the transformed space of feature vectors. Thus, minimizing this term will maximize that fair Bayesian classifier's accuracy on the transformed distribution of features.
The second term measures the correlation between the distribution of feature vectors and the sensitive attribute according to Definition~\ref{decorrelation_def}. Thus, minimizing this term encourages $\mathbf{T}$ to decorrelate the distribution of feature vectors from the sensitive attribute. It can take on a value between $0$ (complete decorrelation) and $2$ (complete correlation). 
The $Fairness$ constraint directly ensures that the fairness constraints from equations $(2)$ and $(3)$ are preserved after the transformation has been applied. 
$f(\mathbf{T})\in\mathbb{R}^{(B+4)\times 1}$ is given by the following equation: 
\begin{equation}
    f(\mathbf{T}) =
    \underbrace{\begin{bmatrix}(\mathbf{p_a}-\mathbf{p_b})^T & \mathbf{0}_{1,N}\\ (\mathbf{p_{a,1}}-\mathbf{p_{b,1}})^T & \mathbf{0}_{1,N} \\ \mathbf{0}_{1,N} & (\mathbf{p_{a,0}}-\mathbf{p_{b,0}})^T \\ (\mathbf{p_a^{1}}-\mathbf{p_b^{1}})^T
     & (\mathbf{p_a^{0}}-\mathbf{p_b^{0}})^T\\
     \mathbf{W} & \mathbf{0}_{B,N}
    \end{bmatrix}}_{\mathbf{P}}
    \underbrace{\begin{bmatrix}\mathbf{T} & \mathbf{0}_{N,N}\\
    \mathbf{0}_{N,N} & \mathbf{T}
    \end{bmatrix}}_{\mathbf{\tilde{T}}}
    \underbrace{\begin{bmatrix} \mathbf{s}^{(F)} \\ \mathbf{1}_{N}-\mathbf{s}^{(F)}
    \end{bmatrix} }_{\mathbf{\tilde{s}}^{(F)}},
\end{equation}
where $\mathbf{\tilde{T}}$ can be directly written as a function of $\mathbf{T}$:
\begin{equation}
    \mathbf{\tilde{T}} = \underbrace{\begin{bmatrix}\mathbf{1}_{N,N} &  \mathbf{0}_{N,N}\\ \mathbf{0}_{N,N} & \mathbf{1}_{N,N}\end{bmatrix}}_{\mathbf{M}}\circ \Biggl( \underbrace{\begin{bmatrix} \mathbf{I}_{N,N} \\ \mathbf{I}_{N,N} \end{bmatrix}}_{\mathbf{\tilde{I}}} \mathbf{T} \underbrace{\begin{bmatrix} \mathbf{I}_{N,N} & \mathbf{I}_{N,N}\end{bmatrix}}_{\mathbf{\tilde{I}}^{T}}\Biggl).
\end{equation}
The first four elements of $|f(\mathbf{T})|$ capture the degree to which a particular group fairness notion is violated in the space of decorrelated feature vectors, while the remaining elements capture the degree to which a pair of neighboring feature vectors from the input space violate local individual fairness after the decorrelation transformation has been applied (see Appendix \ref{trans_deriv} for derivation).
Thus, setting $\mathbf{f}=[\epsilon_{DP}, \ \epsilon_{PE}, \ \epsilon_{EOp}, \ \epsilon_{EOd}, \ \mathbf{1}_{B}^T\epsilon_{IF}]^T$ preserves the group and individual fairness constraints $(2)$ and $(3)$. 

Note that the $Fairness$ constraint can be reformulated as an equality constraint given by $\max(\tilde{f}(\mathbf{T})-\mathbf{\tilde{f}}, \mathbf{0}_{2B+8})=\mathbf{0}_{2B+8}$, where 
\begin{align}
\tilde{f}(\mathbf{T}) = \begin{bmatrix}-\mathbf{P} \\ \mathbf{P}\end{bmatrix} \mathbf{\tilde{T}}\mathbf{\mathbf{\tilde{s}}}^{(F)}
\ \ \ \ \ \ \ 
\text{and}
\ \ \ \ \ \ \ 
\mathbf{\tilde{f}} = \begin{bmatrix}\mathbf{f} \\ \mathbf{f}\end{bmatrix}.
\end{align}
Thus, the question becomes: Given that we must satisfy individual and group fairness constraints associated with the original fair Bayesian classifier's scoring vector, $\mathbf{s}^{(F)}$, how well can we accurately decorrelate the space of feature vectors from the sensitive attribute? To solve this problem, we form the Augmented Lagrangian: 
\begin{align}
\max_{\bm{\rho}}\min_{\mathbf{T}\in\mathcal{P}} 
    &-\lambda(\mathbf{s}^{(F)T}\mathbf{T}\mathbf{p}^{1} + (\mathbf{1}_{N}-\mathbf{s}^{(F)})^T\mathbf{T}\mathbf{p}^{0}) \notag\\
    &+ \beta\|\mathbf{T}(\mathbf{p}_a-\mathbf{p}_b)\|_1 \notag \\
    & + \langle\bm{\rho},\max(\tilde{f}(\mathbf{T})-\mathbf{\tilde{f}}, \mathbf{0}_{2B+8})\rangle \notag \\
    & + \frac{\tau}{2}\|\max(\tilde{f}(\mathbf{T})-\mathbf{\tilde{f}}, \mathbf{0}_{2B+8})\|_2^2 .
\end{align}

\label{results}
\begin{figure*}[t]
\centerline{\includegraphics[width=0.95\textwidth]{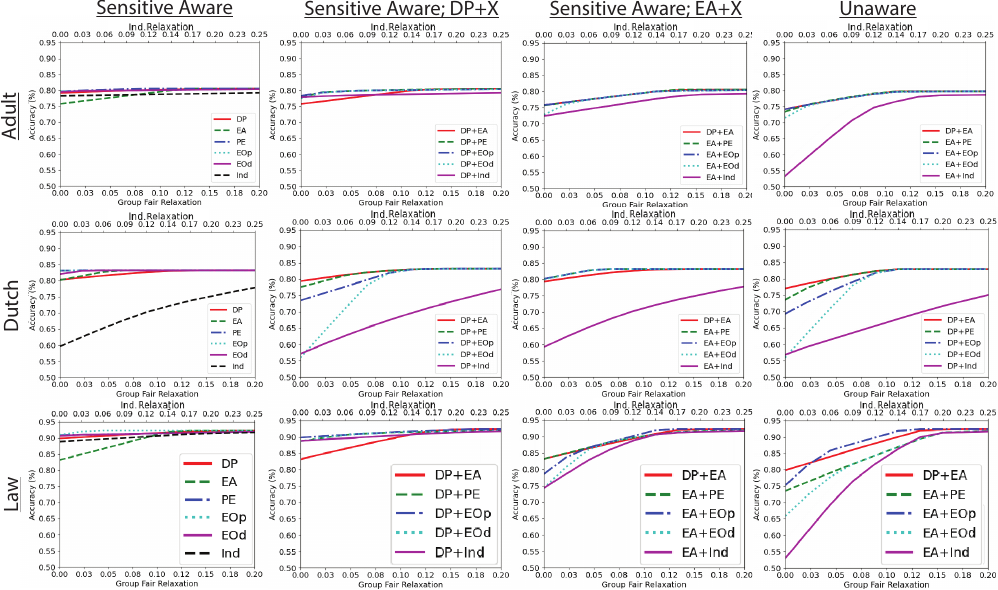}}
\caption{Pareto frontiers capturing the accuracy-fairness trade-off for three datesets under awareness and unawareness of the sensitive attribute. Each plot provides curves for different pairings of fairness constraints; namely, DP, EA, PE, EOd, and Ind.}

\label{Acc_fair_trade}
\end{figure*}

Solving this minimization problem is equivalent to solving minimization problem $(5)$. Thus, if minimization problem $(9)$ is convex, then we can solve it exactly, providing us with the solution to minimization problem $(5)$. Hence, we state the following claim before providing a solution to this problem.
\begin{claim} 
Minimization problem $(9)$ is convex.
\end{claim}
\begin{proof} 
See Appendix \ref{claim1}.
\end{proof}
Let $l(\mathbf{T},\bm{\rho})$ represent the objective function in minimization problem $(9)$. Since this problem is convex, we are able to solve for $\mathbf{T}$ by applying the method of multipliers~\cite{boyd2004convex}. That is, the following updates are applied until convergence where superscript $k$ indicates the value of a variable after the $k^{th}$ iteration:
\begin{align} 
\mathbf{T}^{k+1}&=\argmin_{\mathbf{T}\in\mathcal{P}}l(\mathbf{T},\bm{\rho}^{k}) \notag \\
\bm{\rho}^{k+1}&=\bm{\rho}^{k}+\tau \max(\tilde{f}(\mathbf{T}^{k+1})-\mathbf{\tilde{f}}, \mathbf{0}_{2B+8}).
\end{align}
Optimizing for $\mathbf{T}^{k+1}$ in each iteration can be done using a projected subgradient method, where the projection of $\mathbf{T}$ onto $\mathcal{P}$ is performed by projecting each column of $\mathbf{T}$ onto the unit simplex \cite{duchi2008efficient}. This problem can be generalized to the situation in which the sensitive attribute is allowed to be used to perform the transformation prior to being redacted for classification. 
In this scenario, we solve for two transformation matrices---one for each group. This problem is still convex and can be solved using the alternating direction method of multipliers algorithm \cite{boyd2004convex}; see Appendix \ref{sec_priv_aware} for details. 

\section{Experimental Results}
\label{exp_results}
In this section we experimentally investigate the different modules described in our framework on the Adult \cite{kohavi1996scaling}, Law \cite{wightman1998lsac}, and Dutch Census \cite{van20002001} datasets. Our goals are threefold: (1)~verify the validity of the data created by the generator used to approximate the population distribution of the analyzed datasets, (2)~analyze the trade-off between fairness and accuracy when decorrelation is not a requirement, and (3)~analyze the trade-off between fairness and accuracy when decorrelation is a requirement. The second and third goal are explored under awareness and unawareness of the sensitive attribute. For more details related to our experimental setups, see Appendix \ref{exp_details}.
For details on the time complexities associated with our experiments, see Appendix~\ref{VQ_time}.

\subsection{Verifying the Fidelity of Generator}
\label{generator_fidel}
\begin{table}[h]
    \centering
    \captionsetup{width=\textwidth}
        \caption{PCC and TV Distances between distributions constructed from real and generated data for three datasets. All p-values are below 0.001.}
        \begin{tabular}{cccc|ccc}
            \hline
            & \multicolumn{3}{|c}{PCC $(\uparrow)$} & \multicolumn{3}{|c}{TV Distance $(\downarrow)$}\\
            \hline
            \multicolumn{1}{c}{Distribution} & \multicolumn{1}{|c}{Adult} & \multicolumn{1}{c}{Law} & \multicolumn{1}{c}{Dutch} & \multicolumn{1}{|c}{Adult} & \multicolumn{1}{c}{Law} & \multicolumn{1}{c}{Dutch} \\ 
            \hline
             \multicolumn{1}{c}{$p_{X,A=a,Y=0}(x)$} & \multicolumn{1}{|c}{0.96} & \multicolumn{1}{l}{0.98} & \multicolumn{1}{c}{0.94} & \multicolumn{1}{|c}{0.04} & \multicolumn{1}{l}{0.01} & \multicolumn{1}{c}{0.05}\\
            \multicolumn{1}{c}{$p_{X,A=a,Y=1}(x)$} & \multicolumn{1}{|c}{0.98} & \multicolumn{1}{c}{0.88}& \multicolumn{1}{c}{0.88} & \multicolumn{1}{|c}{0.00} & \multicolumn{1}{c}{0.02}& \multicolumn{1}{c}{0.03}\\
            \multicolumn{1}{c}{$p_{X,A=b,Y=0}(x)$} & \multicolumn{1}{|c}{0.97} & \multicolumn{1}{c}{0.97}& \multicolumn{1}{c}{0.97} & \multicolumn{1}{|c}{0.05} & \multicolumn{1}{c}{0.01}& \multicolumn{1}{c}{0.02}\\
            \multicolumn{1}{c}{$p_{X,A=b,Y=1}(x)$} & \multicolumn{1}{|c}{0.97} & \multicolumn{1}{c}{0.91}& \multicolumn{1}{c}{0.95} & \multicolumn{1}{|c}{0.03} & \multicolumn{1}{c}{0.06}& \multicolumn{1}{c}{0.05}\\
            \hline
        \end{tabular}
    \label{tab_Pear}
\end{table}
The Adult, Law, and Dutch Census datasets respectively contain $48842$, $20798$, and $60420$ samples. Using CT-GAN \cite{xu2019modeling}, we train a generator to learn the underlying distribution of each dataset and use it to produce one million samples for each. Lloyd's algorithm \cite{linde1980algorithm} can then be applied to these samples to construct a discrete approximation of the population distribution. To verify the fidelity of these samples, we use the Pearson Correlation Coefficient (PCC) and Total Variation (TV) Distance to compare the following discrete distributions constructed from the true and generated samples from each dataset: $p_{X,A=a,Y=0}(x), p_{X,A=a,Y=1}(x), p_{X,A=b,Y=0}(x), p_{X,A=b,Y=1}(x)$. We use the bound derived in Appendix \ref{Model} to determine how densely we can quantize our original data, while maintaining the fidelity of the distribution, $(X,A,Y)$, over each cell. Empirically setting the confidence and error parameters to $\delta=0.95$ and $\Delta=0.05$, the bound suggests that we may only partition the true Adult, Law, and Dutch datasets into 48, 20, and 59 VQ cells, respectively. 
Thus, we use these specifications to construct discrete distributions from the true and generated samples for each dataset. Table \ref{tab_Pear} provides the resulting PCC and TV Distance results between the distributions created from the real and generated data. There is good correlation between each distribution constructed from real and generated data, with all values having at least a PCC of 0.88, the majority of which are well above 0.90. Similarly, all TV Distance values are quite low with all values falling below 0.06, though most are less than 0.05. This suggests that the generator has learned the population distribution associated with each dataset, providing us with confidence in using it to construct a more fine-grain cell decomposition from the large number of samples produced by the generator. 

To obtain a more fine-grained cell decomposition for each dataset, we set $N=256$ for all experiments conducted in the body of this paper. For this value, our bound implies that we must produce at least $259849$ samples to faithfully approximate $(X,A,Y)$ with the error parameters $\delta=0.95$ and $\Delta=0.05$. Thus, the one million samples used to construct our discrete approximations of each dataset more than satisfy this requirement.

\subsection{Accuracy-Fairness Trade-off}
In this section, we analyze the fairness-accuracy trade-off when decorrelation between the distribution of feature vectors and the sensitive attribute is not required. We particularly explore various configurations of minimization problem $(4)$ under both awareness and unawareness of the sensitive attribute.

Fig. \ref{Acc_fair_trade} provides a panel of Pareto frontiers that summarize the accuracy-fairness trade-off for different pairings of fairness definitions (see Appendix~\ref{extra_unawareness_pareto} for more analyzed combinations). The plots in each row correspond to one of the three datasets. The first three columns provide results under awareness of the sensitive attribute, while the results in the final column are under unawareness of the sensitive attribute. Along the $x$-axis we relax a group fairness constraint (or pair of constraints). The Individual fairness relaxation budget is scaled differently and shown at the top of each plot along the $x$-axis. The $y$-axis of each plot shows the resulting accuracy of the Bayesian classifier operating under the corresponding fairness relaxation budget. Each curve in a plot corresponds to a different constraint setting. For example, a group fairness relaxation of 0.3 for a DP+EA curve means that $\epsilon_{DP}=\epsilon_{EA}=0.3$, while group fairness relaxation of 0.10 and individual fairness relaxation of 0.12 for a DP+Ind curve means that $\epsilon_{DP}=0.10$ and $\epsilon_{Ind}=0.12$. 
Typically, most group fairness notions can be satisfied exactly in isolation with little accuracy dropoff, as evidenced by the first column of Pareto frontiers. 
This tends to change when two fairness constraints are paired together. The second column couples DP with each of the other fairness notions, while the third column couples EA with each of the other fairness notions. For the Law and Adult datasets, tension can be observed in pairings with EA, while pairings with DP are more easily satisfied. However, the converse is true for the Dutch Census dataset, suggesting that the tension between different group fairness notions is distributionally dependent. The Pareto frontiers in the fourth column of plots are under unawareness of the sensitive attribute. Compared to the situation of awareness, there is a clear deterioration in performance in the Bayesian classifier's accuracy, but the extent of the accuracy dropoff is also distributionally dependent. For example, strictly satisfying EA+PE causes an accuracy reduction of 3\% between the awareness and unawareness situations for the Adult dataset. However, for the Law dataset strictly satisfying EA+PE leads to an accuracy dropoff of 10\% between the awareness and unawareness situations.

\begin{figure}
  \begin{center}
    \includegraphics[width=0.9\textwidth]{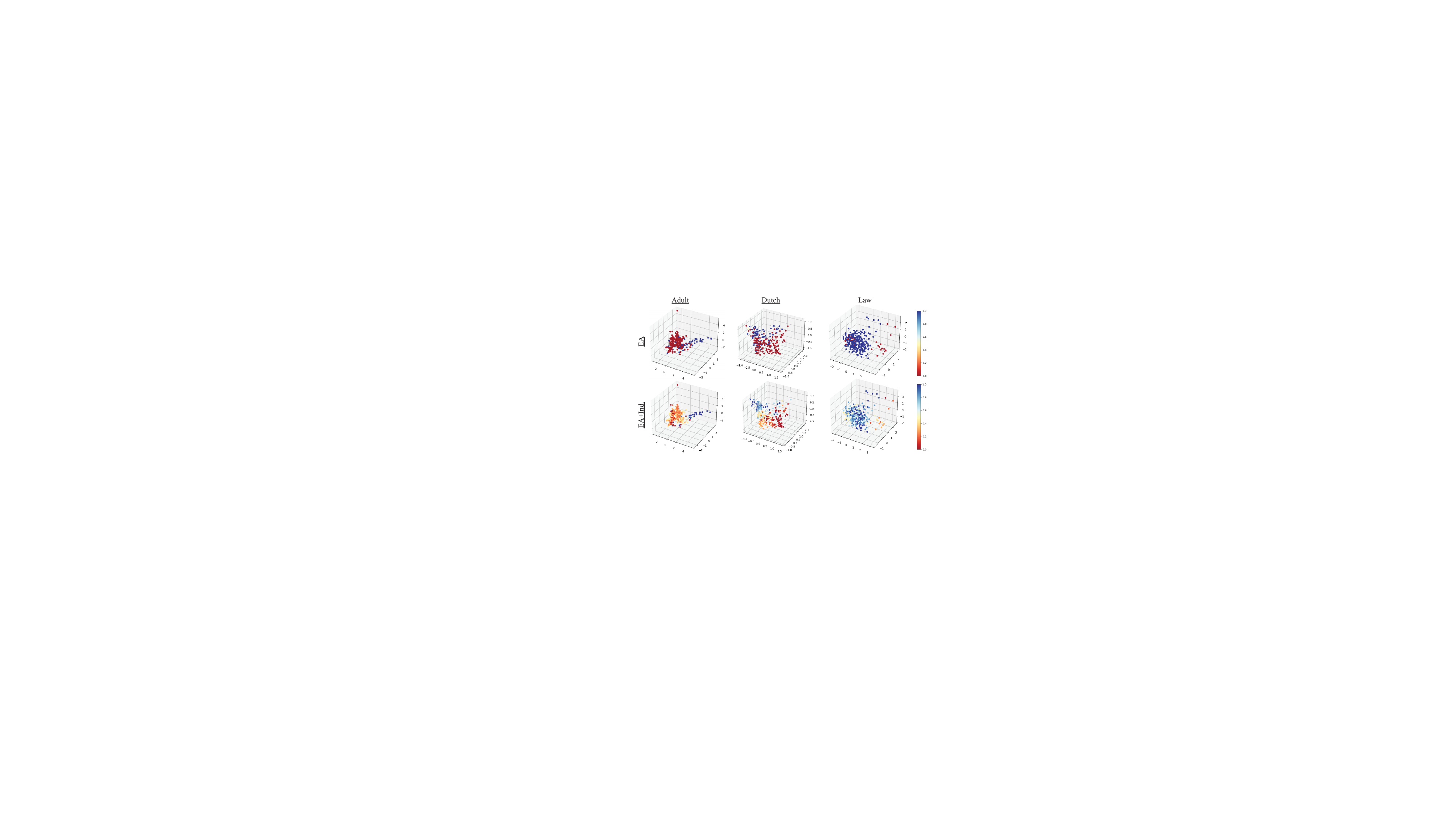}
  \end{center}
  \caption{Dimensionality reduction of feature vectors.}
  \label{Dim_Red}
\end{figure}

Fig. \ref{Acc_fair_trade} also shows that pairing individual fairness with just one group fairness constraint has the potential to considerably reduce the Bayesian classifier's ability to make fair and accurate decisions, particularly in the case of unawareness. Fig. \ref{Dim_Red} provides insight into this phenomenon. In this figure, we produce dimensionality reduction plots, using factor analysis of mixed data (FAMD), to reduce the dimensions of each VQ cell centroid to three \cite{pages2014multiple}. 
Each 3D point is color-coded according to the score of its associated VQ cell (the probability that it receives a positive class label). Points in close proximity to each other represent neighboring VQ cells. The top row of plots displays the results for which the EA constraint is exactly satisfied but without any restrictions on individual fairness. The bottom row of plots also satisfies the EA constraint but with an imposed individual fairness budget of $\epsilon_{IF}=0.15$. In each of the plots in the top row, collections of red and blue points in close proximity to each other can be observed, indicating that VQ cells in close proximity to each other are receiving drastically different scores. The bottom row of plots displays a smoother transition between red and blue points since the individual fairness constraint prohibits the scores that the Bayesian classifier assigns to neighboring VQ cell from being drastically different. Hence, this constraint prohibits the Bayesian classifier from arbitrarily penalizing different VQ cells to satisfy the group fairness constraint.

\subsection{Transfer Fairness to Decorrelated Domain}
\begin{table}[t]
\caption{Results for transferring fairness to decorrelated domain for Adult, Dutch, and Law datasets.}
\label{accuracy_dec_adult}
\centering
\resizebox{\textwidth}{!}{\begin{tabular}{c c   c c  c    c c  c}
 \hline
  & & \multicolumn{3}{c}{Awareness}&  \multicolumn{3}{c}{Unawareness} \\
 \hline
  Dataset&\begin{tabular}{@{}c@{}}Fairness \\ Measure\end{tabular} & \begin{tabular}{@{}c@{}}Baseline \\ Correlation\end{tabular}  & \begin{tabular}{@{}c@{}}Correlation \\ Reduction \end{tabular}$(\uparrow)$ & \begin{tabular}{@{}c@{}}Acc. \\ Reduction \end{tabular}$(\downarrow)$ & \begin{tabular}{@{}c@{}}Baseline \\ Correlation\end{tabular} & \begin{tabular}{@{}c@{}}Correlation \\ Reduction \end{tabular}$(\uparrow)$& \begin{tabular}{@{}c@{}}Acc. \\ Reduction\end{tabular}   $(\downarrow)$\\
 \hline

 &DP+EA  & 2.000 & 2.000 \ (0.000) & 0.005 \ (0.008) & 0.920 & 0.920 \ (0.000) & 0.007 \ (0.008)\\
 &DP+EOd  & & 2.000 \ (0.000) &  0.008 \ (0.007) & & 0.920 \ (0.000) & 0.010 \ (0.007)\\
 Adult&EA+EOd   & & 1.939 \ (0.084) & 0.010 \ (0.012) & & 0.870 \ (0.080)& 0.012 \ (0.005)\\
 &DP+Ind   & & 2.000 \ (0.000)& 0.014 \ (0.008) & & 0.920 \ (0.000) & 0.003 \ (0.002)\\
 &EA+Ind & & 1.974 \ (0.044) & 0.021 \ (0.006) & & 0.920 \ (0.000) & 0.003 \ (0.000)\\
 &EOd+Ind & & 2.000 \ (0.000) & 0.015 \ (0.004) & & 0.918 \ (0.003) & 0.003 \ (0.002)\\
 \hline
 
 &DP+EA  & 2.000 & 2.000 \ (0.000) & 0.017 \ (0.016)  & 0.560 &  0.544 (0.019) & 0.025 \ (0.023)\\
 &DP+EOd & & 2.000 \ (0.000) & 0.009 \ (0.016) & & 0.498 \ (0.053) & 0.015 \ (0.022)\\
 Dutch & EA+EOd & & 1.801 \ (0.099) & 0.021 \ (0.019)  & & 0.367 \ (0.088) & 0.019 \ (0.017)\\
 &DP+Ind & & 2.000 \ (0.000) &  0.020 \ (0.008) & & 0.560 \ (0.000) & 0.009 \ (0.007)\\
 &EA+Ind & & 1.989 \ (0.032) &  0.061 \ (0.011) & & 0.560 \ (0.000) & 0.012 \ (0.002)\\
 &EOd+Ind & & 1.980 \ (0.035) & 0.022 \ (0.009) &  & 0.542 \ (0.031)& 0.011 \ (0.006)\\
 \hline
 &DP+EA  & 2.000 & 2.000 \ (0.000) & 0.021 \ (0.018) & 0.980 & 0.980 \ (0.000) & 0.026 \ (0.025)\\
 &DP+EOd & & 2.000 \ (0.000) & 0.020 \ (0.018) & & 0.980 \ (0.000) & 0.014 \ (0.013)\\
 Law&EA+EOd & & 1.880 \ (0.113) & 0.019 \ (0.026) & & 0.910 \ (0.074) & 0.010 \ (0.012)\\
 &DP+Ind  & & 2.000 \ (0.000) & 0.021 \ (0.013) & & 0.980 \ (0.000) & 0.005 \ (0.005)\\
 &EA+Ind & & 2.000 \ (0.000) &  0.056 \ (0.006) & & 0.980 \ (0.000) & 0.004 \ (0.001)\\
 &EOd+Ind & & 1.977 \ (0.021) & 0.020 \ (0.006) &  & 0.980 \ (0.000) & 0.005 \ (0.004)\\
 \hline
\end{tabular}}
\end{table}

In this section, we analyze the extent to which the space of feature vectors can be decorrelated with respect to the sensitive attribute while preserving the fairness of the Bayesian classifier’s decisions. In our experiments, we set the hyperparameters in minimization problem  $(9)$ to 
$\lambda=15$ and $\beta=25$. Table \ref{accuracy_dec_adult} displays the results obtained from our decorrelation analysis under the preservation of different combinations of fairness definitions for the Adult, Dutch, and Law datasets.
\begin{figure}
  \begin{center}
    \includegraphics[width=0.99\textwidth]{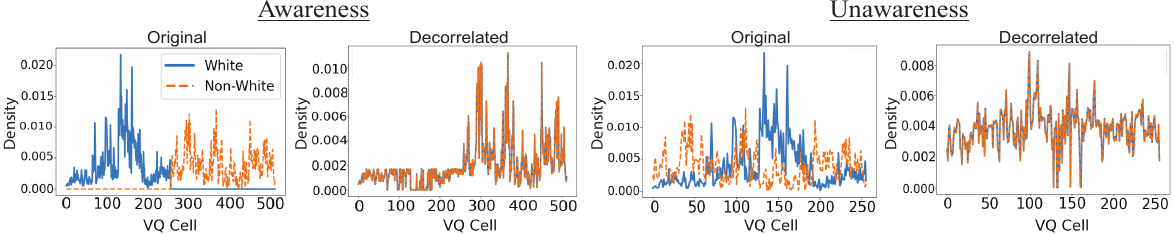}
  \end{center}
  \caption{Visualization of the original and decorrelated feature vector distributions when the sensitive attribute is and is not used for decorrelation.} 
  \label{priv_plots}
\end{figure}

For all combinations involving individual fairness, we hold $\epsilon_{Ind}=0.05$, meaning that the deviation in the probability of neighboring feature vectors being assigned a positive class label by the Bayesian classifier should be no more than $5$\%. All group fairness constraints are tested for relaxations of $0.0, 0.5,$ and $0.10$. Thus, we report the average over three values in each of the cells of this table along with their standard deviation in parentheses. We report the correlation between the distribution of feature vectors and the sensitive attribute prior to appyling the decorrelation mapping for each dataset under "Baseline Correlation." The reduction in correlation achieved from applying the the decorrelation mapping is given by "Correlation Reduction." As these values become closer to the value of the baseline correlation, the distribution of feature vectors becomes more decorrelated with the sensitive attribute. We further report the accuracy dropoffs that result from using the decorrelated distribution of feature instead of the original distribution for fair classification under "Acc. Reduction."
The results in this table suggest that it is possible to decorrelate the space of feature vectors with little accuracy drop-off in situations in which the sensitive attribute is and is not available for constructing the decorrelation mapping. That is, for all datasets, the accuracy of the fair Bayesian classifier on the decorrelated feature vectors typically falls by around $2$\% or less on average, with the distribution of feature vectors becoming almost completely decorrelated with the sensitive attribute in all cases except for the EA+EOd pairing. The EA+EOd pairing is the only one that struggles to come close to achieving complete decorrelation between the distribution of feature vectors and the sensitive attribute for all datasets. However, even for this pairing, the transformed feature vectors are much more decorrelated with respect to the sensitive attribute compared to the non-transformed feature vectors.

We illustrate the effect of these decorrelation mappings under awareness and unawareness of the sensitive attribute in Fig. \ref{priv_plots}, plotting the conditional distributions $P(X|A=\text{White})$ and $P(X|A=\text{Non-White})$ over all VQ cells for the Law dataset. 
Under awareness of the sensitive attribute, the number of cells along the $x$-axis doubles compared to unawareness since access to the sensitive attribute allows the Bayesian classifier to separate each cell into two decision regions. 
These plots demonstrate the effectiveness of the decorrelation mappings with the density of feature vectors in each VQ cell perfectly overlapped in the awareness and unawareness plots in which decorrelation has been applied. This suggests that the accuracy drop-off of a fair classifier is less dependent on how correlated the features are with the sensitive attribute and more dependent on the strictness of the fairness enforcement.



\section{Conclusion, Broader Impact, and Limitations}
\label{conclusion}
This paper explores the trade-off between fairness and accuracy under four practical scenarios that limit the data available for classification. We investigate the behavior of a fair Bayesian classifier by approximating the joint distribution of the feature vector, sensitive attribute, and class label. Our exploration encompasses situations in which the sensitive attribute may or may not be available and correlations between feature vectors and the sensitive attribute may or may not be eliminated. 
Our results also suggest that the pursuit of individual fairness through the enforcement of local scoring consistency may clash with notions of group fairness, particularly when the sensitive attribute is unavailable to a model. Additionally, we demonstrate that it is often feasible to reduce the correlation between the space of feature vectors and the sensitive attribute while preserving the accuracy of a fair model.

\textbf{Broader Impact} Fairness in machine learning is a problem of increasing relevance in today's society given the huge increase in applications that use such technology for decision-making. Feasibility studies like ours, which analyze the fairness of such models for real world application scenarios, thus become critical for developing ethical technology. Our study can help developers of fair ML models save time and resources by providing them with a useful tool to understand the extent to which it is even possible for them to produce such models.

\textbf{Limitations} We would like to mention two limitations with regards to our approach. First, a sizeable enough dataset for training the generator is needed to model the population distribution of the feature vector, sensitive attribute, and class label. This is, however, not a problem unique to our approach, but relevent to all machine learning applications that rely on the substance of the data used for training.  Our experimentation has shown that a dataset of approximately 20000 samples suffices to model this distribution effectively.
Second, quantifying the similarity among individuals introduces subjectivity. The choice of distance metric and the specific parameters used in the metric inherently imply certain assumptions about the similarity between two individuals. Thus, adjusting the distance metric and parameters may produce different results. Noteably, this issue is not particular to our framework, but rather a broader challenge inherent in the analysis of individual fairness. 


\section{Disclaimer}
This paper was prepared for informational purposes in part by the Artificial Intelligence Research group of JPMorgan Chase \& Co. and its affiliates (“JP Morgan”), and is not a product of the Research Department of JP Morgan. JP Morgan makes no representation and warranty whatsoever and disclaims all liability, for the completeness, accuracy or reliability of the information contained herein. This document is not intended as investment research or investment advice, or a
recommendation, offer or solicitation for the purchase or sale of any security, financial instrument, financial product or service, or to be used in any way for evaluating the merits of participating in any transaction, and shall not constitute a solicitation under any jurisdiction or to any person, if such solicitation under such jurisdiction or to such person would be unlawful.

\bibliographystyle{ACM-Reference-Format}
\bibliography{main}

\appendix

\section{Modeling the Population Distribution}
\label{Model}
Let $X, A, \text{ and } Y$ represent the random feature vector and  sensitive attribute and class label random variables respectively associated with the sample spaces $\mathcal{X}, \mathcal{A}, \text{ and } \mathcal{Y}$. For simplicity, we assume that $\mathcal{X}=\mathbb{R}^k, \mathcal{A}=\{a,b\}, \text{ and } \mathcal{Y}=\{0,1\}$. 
Let $D^{(tr)} = \{\mathbf{x}_i^{(tr)},g_i^{(tr)},y_i^{(tr)}|\mathbf{x}_i^{(tr)}\in\mathcal{X},g_i^{(tr)}\in\mathcal{A},y_i^{(tr)}\in\mathcal{Y}\}_{i=1}^{T}$ 
represent a \textit{dataset}, containing $T$ samples from the joint distribution $(X,A,Y)$. A \textit{generator}, $G:\mathcal{L}\rightarrow \mathcal{X\times Y \times A}$, learns the latent representation of $(X,A,Y)$ from $D^{(tr)}$ and can be used to sample the distribution $(X, A, Y)$. Let $D^{(ts)}=\{\mathbf{x}^{(ts)}_i,g^{(ts)}_i,y^{(ts)}_i|\mathbf{x}^{(ts)}_i\in\mathcal{X},g^{(ts)}_i\in\mathcal{A},y^{(ts)}_i\in\mathcal{Y}\}_{i=1}^{M}$ be a dataset produced by generating $M$ samples using $G$. \\

\begin{definition}
(Cell decomposition)
A cell decomposition of size $N$ over $\mathcal{X}$ is given by a disjoint set of $N$ cells, $\mathcal{C}_1,...,\mathcal{C}_{N} \subseteq \mathcal{X}$, that cover $\mathcal{X}$. These cells are defined by their centroids, $\mathbf{x}_1,...,\mathbf{x}_N\in \mathcal{X}$,  where for each $i\in \mathbb{N}_{\leq N}^+$ and for some distance metric, $d_\mathcal{X}$:\\
\begin{align}
    \mathcal{C}_i=\{\mathbf{k}|\mathbf{k}\in\mathcal{X}, d_\mathcal{X}(\mathbf{x}_i,\mathbf{k}) \leq d_\mathcal{X}(\mathbf{x}_j,\mathbf{k}), \forall j\in \mathbb{N}_{\leq N}^+\setminus i\}. \notag
\end{align}   
\end{definition}

We construct a discrete approximation of $(X,A,Y)$, represented by $(\tilde{X},\tilde{A},\tilde{Y})$, by inducing a cell decomposition over $\mathcal{X}$, and using $D$ to quantify the statistical properties of $(\tilde{X},\tilde{A},\tilde{Y})$ over each cell. Specifically, let $\mathcal{C}=\{\mathbf{x}_i |\mathbf{x}_i\in \mathcal{X}\}_{i=1}^{N}$ be the set containing the $N$ centroids produced by a cell decomposition.
We construct the probability mass function (pmf) of $\tilde{X}$ as follows: 
\begin{equation}
    p_{\tilde{X}}(\mathbf{k}) = \begin{cases}
      \frac{M_i}{M} &\text{, if } \exists i\in\mathbb{N}_{\leq N}^+ \ \text{ s.t. } \ \mathbf{k}=\mathbf{x}_i\\
      0 &\text{, otherwise}
    \end{cases},
\end{equation}
where $M_i$ represents the number of feature vectors from the dataset $D$ in cell $\mathcal{C}_i$.
The conditional joint distribution, $\tilde{Y},\tilde{A}|\tilde{X}=\mathbf{x}_i$, is given by the portion of samples from $D^{(ts)}$ with feature vectors in cell $C_i$ for which $g$ and $y$ are their sensitive attribute and class label. That is,
\begin{equation}
    p_{\tilde{A},\tilde{Y}|\tilde{X}=\mathbf{x}_i}(g,y) = \frac{\sum_{j=1}^MI[\mathbf{x}_j^{(ts)}\in \mathcal{C}_i,g_j^{(ts)}=g,y_j^{(ts)}=y ]}{M_i},
\end{equation}
where $I$ represents the indicator function. Thus, equations (11) and (12) provide us with the joint distribution $(\tilde{X},\tilde{A},\tilde{Y})$.

Since all of our analyses are performed on the approximated distribution, we must sample $G$ densely enough to guarantee that the statistical information inside of each cell of this distribution accurately characterizes the information in each cell of the true population distribution.
The following observation provides us with a guide for sampling $G$.
\begin{observation} 
\label{obs1}
(PAC Bound for sampling generator) Suppose we induce a cell decomposition of size $N$ over $\mathcal{X}$ using Lloyd's algorithm for vector quantization \cite{lloyd1982least,linde1980algorithm}. Let $j$ refer to an arbitrary cell index. Define $\nu_{j,g}^y=P(A=g,Y=y|X\in C_j)$ and $\mu_{j,g}^y=P(\tilde{A}=g,\tilde{Y}=y|\tilde{X}=\mathbf{x}_j)$. Then, for some $\Delta \text{ and } \delta$, using a total of $M=\frac{-N}{2\Delta^2}ln(\frac{1-\delta}{8})$ samples from $G$ will guarentee that $P(\cup_{g\in\mathcal{A},y\in\mathcal{Y}}|\nu_{j,g}^y-\mu_{j,g}^y| \geq \Delta) \leq 1-\delta$ on average for all cells.
\end{observation}
\begin{proof}
    Consider a series of 4-dimensional random vectors, $B_1,...,B_n$, where each random variable in a vector has a Bernoulli distribution characterized by the probability of selecting a feature vector that has class label $Y=y$ and group label $A=g$ or not from cell $\mathcal{C}_j$, where $g\in\mathcal{A},y\in\mathcal{Y}$. The sample mean and population means of these vectors are given by $\mu=[\mu_{j,a}^0,\mu_{j,a}^1,\mu_{j,b}^0, \mu_{j,b}^1]^T$ and $\nu=[\nu_{j,a}^0,\nu_{j,a}^1,\nu_{j,b}^0, \nu_{j,b}^1]^T$, respectively. Thus, we have that:
\begin{align}
    &P(\cup_{g,y}(|\mu_{j,g}^y-\nu_{j,g}^y|\geq \Delta))\\
    &\leq \sum_{g,y}P(|\mu_{j,g}^y-\nu_{j,g}^y|\geq \Delta) \text{ ( Union Bound) }\notag\\
    &\leq 8e^{-2\Delta^2s_j} \text{ (Hoeffding Inequality) },
\end{align}
where $s_j$ refers to the number of samples inside cell $\mathcal{C}_j$. Lloyd's algorithm asymptotically induces a cell decomposition in which the density in each cell is equal \cite{lloyd1982least}. Thus, we assume $s_j=N/N$. Now, we can average over (13) and (14) to get:
\begin{align}
    &\frac{1}{N}\sum_{j}P(\cup_{g,y}(|\mu_{j,g}^y-\nu_{j,g}^y|\geq \Delta))\leq \frac{\cancel{N}}{\cancel{N}}8e^{-\frac{2\Delta^2M}{N}}
\end{align}
Setting $8e^{-\frac{2\Delta^2M}{N}}=1-\delta$, and solving for $M$ completes this proof.
\end{proof}
Thus, Observation \ref{obs1} provides us with a guide for how densely we should sample $G$ to guarantee the fidelity of the joint distribution of the class label and sensitive attribute over each cell.

\section{Derivation of Minimization Problem (4)}
\label{app_LP_const}
Let $S^{*}:\mathcal{X}\rightarrow[0,1]$ and $S^{(F)}:\mathcal{X}\rightarrow[0,1]$ respectively represent the optimal unconstrained and fair randomized scoring functions whose outputs represent the probability of producing a $1$ label for a given feature vector. Furthermore, let
\begin{equation}
    \hat{Y}^{*}(\mathbf{x})=\begin{cases}
        1&\text{, w.p. } S^{*}(\mathbf{x})\\
        0&\text{, w.p. } 1-S^{*}(\mathbf{x})
    \end{cases}
    \ \ \ \ \ \ \ \text{and}  \ \ \ \ \ \ \
    \hat{Y}^{(F)}(\mathbf{x})=\begin{cases}
        1&\text{, w.p. } S^{(F)}(\mathbf{x})\\
        0&\text{, w.p. } 1-S^{(F)}(\mathbf{x})
    \end{cases}
    \notag
\end{equation}
With this information, we derive the objective function in minimization problem (4), which represents the reduction in accuracy between the unconstrained and fair Bayesian classifiers, by showing that it is equivalent to:
\begin{align}
    P(\hat{Y}^{*}(X)=Y)-P(\hat{Y}^{(F)}(X)=Y) \notag
\end{align}
We further derive the group fairness linear constraints provided in minimization problem (4) by showing that they are equivalent the the following list of constraints:
\begin{align}
    & |P(\hat{Y}^{(F)}(X)=1|A=a)-P(\hat{Y}^{(F)}(X)=1|A=b)|\leq\epsilon_{DP} \notag \\
    & |P(\hat{Y}^{(F)}(X)=1|A=a,Y=1)-P(\hat{Y}^{(F)}(X)=1|A=b,Y=1)|\leq\epsilon_{EOp} \notag \\
    & |P(\hat{Y}^{(F)}(X)=1|A=a,Y=0)-P(\hat{Y}^{(F)}(X)=1|A=b,Y=0)|\leq\epsilon_{PE} \notag \\
    & |P(\hat{Y}^{(F)}(X)=Y|A=a)-P(\hat{Y}^{(F)}(X)=Y|A=b)|\leq\epsilon_{EA}, \notag
\end{align}
Note that equalized odds is omitted since it is simply a combination of predictive equality and equal opportunity. We also omit the derivation of the individual fairness constraints since they are constructed in a straightforward manner in the body of the paper.
\subsection{Objective Function:} Observe that:
\begin{align}
    &P(\hat{Y}^{*}(X)=Y)-P(\hat{Y}^{(F)}(X)=Y) = \sum_{i=1}^{N}\left[ P(\hat{Y}^{*}(X)=Y, X=\mathbf{x}_i)-P(\hat{Y}^{(F)}(X)=Y,X=\mathbf{x}_i)\right]\notag\\
    &=\sum_{i=1}^{N}\Bigl[P(\hat{Y}^{*}(X)=1,Y=1, X=\mathbf{x}_i)+P(\hat{Y}^{*}(X)=0,Y=0, X=\mathbf{x}_i) \notag
\end{align}
\begin{align}
    & \ \ \ \ \ \ \ \ \ \ \ \ \ \ -P(\hat{Y}^{(F)}(X)=1,Y=1, X=\mathbf{x}_i)-P(\hat{Y}^{(F)}(X)=0,Y=0, X=\mathbf{x}_i)\Bigr]\notag\\
    &=\sum_{i=1}^{N}\Bigl[\{P(\hat{Y}^{*}(X)=1|\underbrace{\cancel{Y=1}}_{\substack{\text{\makebox[0pt]{Conditional independence given $X$.}} \\ \text{\makebox[0pt]{Similar for other cancellations.}}}}, X=\mathbf{x}_i) - P(\hat{Y}^{(F)}(X)=1|\cancel{Y=1}, X=\mathbf{x}_i)\}\underbrace{P(Y=1, X=\mathbf{x}_i)}_{\substack{\text{\makebox[0pt]{$=\mathbf{p}^{1}[i]$ by definition.}} \\ \text{\makebox[0pt]{Similar for other substitutions.}}}} \notag\\
    & \ \ \ \ \ \ \ \ \ \ \ \ \ \ +\{P(\hat{Y}^{*}(X)=0|\cancel{Y=0}, X=\mathbf{x}_i) - P(\hat{Y}^{(F)}(X)=0|\cancel{Y=0},X=\mathbf{x}_i)\}P(Y=0,X=\mathbf{x}_i)\Bigr]\notag\\
    &=\sum_{i=1}^{N}\Bigl[\{\underbrace{P(\hat{Y}^{*}(X)=1| X=\mathbf{x}_i)}_{\substack{\text{\makebox[0pt]{$=\mathbf{s}^{*}[i]$ by definition.}} \\ \text{\makebox[0pt]{Similar for other substitutions.}}}} - P(\hat{Y}^{(F)}(X)=1| X=\mathbf{x}_i)\} \mathbf{p}^1[i] \notag\\
    & \ \ \ \ \ \ \ \ \ \ \ \ \ \ +\{P(\hat{Y}^{*}(X)=0| X=\mathbf{x}_i) - P(\hat{Y}^{(F)}(X)=0|X=\mathbf{x}_i)\}\mathbf{p}^0[i]\Bigr]\notag\\
    &=\sum_{i=1}^{N}\Bigl[\{\mathbf{s}^{*}[i] - \mathbf{s}^{(F)}[i]\} \mathbf{p}^1[i] +\{1-\mathbf{s}^{*}[i] - 1+\mathbf{s}^{(F)}[i]\}\mathbf{p}^0[i]\Bigr]\notag\\
    &=\sum_{i=1}^{N}\Bigl[(\mathbf{s}^{*}[i] - \mathbf{s}^{(F)}[i])(\mathbf{p}^1[i] -\mathbf{p}^0[i])\Bigr]\underbrace{=}_{(A)}\sum_{i=1}^{N}\mathbf{m}[i](\mathbf{p}^1[i] -\mathbf{p}^0[i])=\mathbf{m}^{T}(\mathbf{p}^1 -\mathbf{p}^0) \notag
\end{align}
where $(A)$ holds since
\begin{align}
    &\sum_{i=1}^{N}\Bigl[(\mathbf{s}^{*}[i] - \mathbf{s}^{(F)}[i])(\mathbf{p}^1[i] -\mathbf{p}^0[i])\Bigr]\notag\\
    &=\sum_{i\in\mathcal{I}^+}\Bigl[(\underbrace{\mathbf{s}^{*}[i] - \mathbf{s}^{(F)}[i]}_{\text{\makebox[0pt]{$\geq0$ since $\mathbf{s}^{*}[i]=1$ for $i\in\mathcal{I}^+$}}})(\mathbf{p}^1[i] -\mathbf{p}^0[i])\Bigr]+\sum_{i\in\mathcal{I}^-}\Bigl[(\underbrace{\mathbf{s}^{*}[i] - \mathbf{s}^{(F)}[i]}_{\text{\makebox[0pt]{$\leq0$ \ since $\mathbf{s}^{*}[i]=0$ for $i\in\mathcal{I}^-$}}})(\mathbf{p}^1[i] -\mathbf{p}^0[i])\Bigr]\notag
\end{align}
\begin{align}
    &=\sum_{i\in \mathcal{I}^+}\Bigl[\underbrace{|\mathbf{s}^{*}[i] - \mathbf{s}^{(F)}[i]|}_{\text{\makebox[0pt]{$\mathbf{m}[i]$ by definition for $i\in \mathcal{I}^+$.}}}(\mathbf{p}^1[i] -\mathbf{p}^0[i])\Bigr]+\sum_{i\in\mathcal{I}^-}\Bigl[\underbrace{-|\mathbf{s}^{*}[i] - \mathbf{s}^{(F)}[i]|}_{\text{\makebox[0pt]{$\mathbf{m}[i]$ by definition for $i\in\mathcal{I}^-$}}}(\mathbf{p}^1[i] -\mathbf{p}^0[i])\Bigr]\notag\\
    &=\sum_{i\in\mathcal{I}^+}\mathbf{m}[i](\mathbf{p}^1[i] -\mathbf{p}^0[i])+\sum_{i\in\mathcal{I}^-}\mathbf{m}[i](\mathbf{p}^1[i] -\mathbf{p}^0[i])=\sum_{i=1}^{N}\mathbf{m}[i](\mathbf{p}^1[i] -\mathbf{p}^0[i]).\notag
\end{align}

\subsection{Demographic Parity:} 
Observe that:
\begin{align}
    &P(\hat{Y}^{(F)}(X)=1|A=a) = \sum_{i=1}^{N}P(\hat{Y}^{(F)}(X)=1,X=\mathbf{x}_i|A=a) \notag\\
    &= \sum_{i\in\mathcal{I}^+}P(\hat{Y}^{(F)}(X)=1,X=\mathbf{x}_i|A=a) + \sum_{i\in\mathcal{I}^-}P(\hat{Y}^{(F)}(X)=1,X=\mathbf{x}_i|A=a)\notag
\end{align}
\begin{align}
    &= \sum_{i\in\mathcal{I}^+}\overbrace{P(\hat{Y}^{(F)}(X)=1|X=\mathbf{x}_i,\underbrace{\cancel{A=a}}_{\substack{\text{\makebox[0pt]{Conditional independence given $X$.}} \\ \text{\makebox[0pt]{Similar for other cancellations.}}}})}^{\mathbf{s}^{(F)}[i]=\mathbf{s}^{*}[i]-(\mathbf{s}^{*}[i]-\mathbf{s}^{(F)}[i]) \text{ by definition.}}P(X=\mathbf{x}_i|A=a) \notag\\
    & \ \ \ \ \ \ \ \ \ \ \ \ \ \ + \sum_{i\in\mathcal{I}^-}\underbrace{P(\hat{Y}^{(F)}(X)=1|X=\mathbf{x}_i,\cancel{A=a})}_{\substack{\mathbf{s}^{(F)}[i]=\mathbf{s}^{*}[i]-(\mathbf{s}^{*}[i]-\mathbf{s}^{(F)}[i]) \text{ by definition }}}P(X=\mathbf{x}_i|A=a)\notag\\
    &=\sum_{i\in\mathcal{I}^+}[\mathbf{s}^{*}[i]-\underbrace{(\mathbf{s}^{*}[i]-\mathbf{s}^{(F)}[i])}_{\substack{\text{\makebox[0pt]{$|\mathbf{s}^{*}[i]-\mathbf{s}^{(F)}[i]|$}} \\ \makebox[0pt]{$=\mathbf{m}[i]$} \\ \text{\makebox[0pt]{by definition for $i\in \mathcal{I}^+$.} }}}]\underbrace{P(X=\mathbf{x}_i|A=a)}_{\substack{\text{\makebox[0pt]{$\mathbf{p}_a[i]$ by definition.}}}} + \sum_{i\in\mathcal{I}^-}[\mathbf{s}^{*}[i]-\underbrace{(\mathbf{s}^{*}[i]-\mathbf{s}^{(F)}[i])}_{_{\substack{\text{\makebox[0pt]{$-|\mathbf{s}^{*}[i]-\mathbf{s}^{(F)}[i]|$}} \\ \makebox[0pt]{$=\mathbf{m}[i]$} \\ \text{\makebox[0pt]{by definition for $i\in \mathcal{I}^-$.} }}}}]P(X=\mathbf{x}_i|A=a)\notag\\
    & =\sum_{i=1}^{N}(\mathbf{s}^{*}[i]-\mathbf{m}[i])\mathbf{p}_a[i] \notag \\
    & =(\mathbf{s}^{*}-\mathbf{m})^T\mathbf{p}_a \notag
\end{align}
A similar derivation of $P(\hat{Y}^{(F)}(X)=1|A=b)$ yields: $P(\hat{Y}^{(F)}(X)=1|A=b) = (\mathbf{s}^{*}-\mathbf{m})^T\mathbf{p}_b$. Thus, 
\begin{align}
    |P(\hat{Y}^{(F)}(X)=1|A=a)-P(\hat{Y}^{(F)}(X)=1|A=b)|\leq\epsilon_{DP} \notag
\end{align} 
is equivalent to 
\begin{align}
    |(\mathbf{s}^{*}-\mathbf{m})^T(\mathbf{p}_a-\mathbf{p}_b)|\leq \epsilon_{DP} 
    \ \ \ \ \ \ \ \ \text{ or } \ \ \ \ \ \ \ \
    |(\mathbf{p}_a-\mathbf{p}_b)^T(\mathbf{s}^{*}-\mathbf{m})|\leq \epsilon_{DP}.\notag
\end{align}

\subsection{Equal Opportunity:} Observe that:
\begin{align}
    &P(\hat{Y}^{(F)}(X)=1|A=a,Y=1) = \sum_{i=1}^{N}P(\hat{Y}^{(F)}(X)=1,X=\mathbf{x}_i|A=a,Y=1) \notag\\
    &= \sum_{ i\in \mathcal{I}^+ }P(\hat{Y}^{(F)}(X)=1,X=\mathbf{x}_i|A=a,Y=1) + \sum_{i\in \mathcal{I}^-}P(\hat{Y}^{(F)}(X)=1,X=\mathbf{x}_i|A=a,Y=1)\notag\\
    &= \sum_{i\in \mathcal{I}^+}\overbrace{P(\hat{Y}^{(F)}(X)=1|X=\mathbf{x}_i,\underbrace{\cancel{A=a,Y=1}}_{\substack{\text{\makebox[0pt]{Conditional independence given $X$.}} \\ \text{\makebox[0pt]{Similar for other cancellations.}}}})}^{\mathbf{s}^{(F)}[i]=\mathbf{s}^{*}[i]-(\mathbf{s}^{*}[i]-\mathbf{s}^{(F)}[i]) \text{ by definition.}}P(X=\mathbf{x}_i|A=a,Y=1) \notag
\end{align}
\begin{align}
    &\ \ \ \ \ \ \ \ \ + \sum_{i\in \mathcal{I}^-}\underbrace{P(\hat{Y}^{(F)}(X)=1|X=\mathbf{x}_i,\cancel{A=a,Y=1})}_{\mathbf{s}^{(F)}[i]=\mathbf{s}^{*}[i]-(\mathbf{s}^{*}[i]-\mathbf{s}^{(F)}[i]) \text{ by definition.}}P(X=\mathbf{x}_i|A=a,Y=1)\notag\\
    &=\sum_{i\in \mathcal{I}^+}\left[\mathbf{s}^{*}[i]-(\mathbf{s}^{*}[i]-\mathbf{s}^{(F)}[i])\right]\underbrace{P(X=\mathbf{x}_i|A=a,Y=1)}_{\substack{\text{\makebox[0pt]{$=\mathbf{p}_{a,1}[i]$ by definition.}} \\ \text{\makebox[0pt]{Similar for other substitutions.}}}} \notag\\
    & \ \ \ \ \ \ \ \ \ + \sum_{i\in \mathcal{I}^-}\left[\mathbf{s}^{*}[i]-(\mathbf{s}^{*}[i]-\mathbf{s}^{(F)}[i]) \right]P(X=\mathbf{x}_i|A=a,Y=1)\notag\\
    &\underbrace{=}_{\text{\makebox[0pt]{Similar to the demographic parity derivation.}}}(\mathbf{s}^{*}-\mathbf{m})^T\mathbf{p}_{a,1} \notag
\end{align}
A similar derivation of $P(\hat{Y}^{(F)}(X)=1|A=b,Y=1)$ yields: $P(\hat{Y}^{(F)}(X)=1|A=b,Y=1) = (\mathbf{s}^{*}-\mathbf{m})^T\mathbf{p}_{b,1}$. Thus, 
\begin{align}
    |P(\hat{Y}^{(F)}(X)=1|A=a,Y=1)-P(\hat{Y}^{(F)}(X)=1|A=b,Y=1)|\leq\epsilon_{EOp} \notag
\end{align}
is equivalent to 
\begin{align}
    |(\mathbf{s}^{*}-\mathbf{m})^T(\mathbf{p}_{a,1}-\mathbf{p}_{b,1})|\leq \epsilon_{EOp}
    \ \ \ \ \ \ \ \ \text{ or } \ \ \ \ \ \ \ \
    |(\mathbf{p}_{a,1}-\mathbf{p}_{b,1})^T(\mathbf{s}^{*}-\mathbf{m})|\leq \epsilon_{EOp}. \notag
\end{align}

\subsection{Predictive Equality}
This derivation is closely related to the Equal Opportunity derivation. Hence we omit it, directly claiming that
\begin{equation}
    |P(\hat{Y}^{(F)}(X)=1|A=a,Y=0)-P(\hat{Y}^{(F)}(X)=1|A=b,Y=0)|\leq\epsilon_{PE} \notag
\end{equation}
is equivalent to 
\begin{equation}
    |(\mathbf{p}_{a,0}-\mathbf{p}_{b,0})^T(\mathbf{s}^{*}-\mathbf{m})|\leq\epsilon_{PE}. \notag
\end{equation}

\subsection{Equal Accuracy:} Observe that:
\begin{align}
    &P(\hat{Y}^{(F)}(X)=Y|A=a) = P(\hat{Y}^{(F)}(X)=1,Y=1|A=a) + P(\hat{Y}^{(F)}(X)=0,Y=0|A=a)\notag
\end{align}
\begin{align}
    &=\sum_{i=1}^{N}P(\hat{Y}^{(F)}(X)=1,Y=1,X=\mathbf{x}_i|A=a) + P(\hat{Y}^{(F)}(X)=0,Y=0,X=\mathbf{x}_i|A=a) \notag
\end{align}
\begin{align}
    &= \sum_{i=1}^{N}P(\hat{Y}^{(F)}(X)=1|X=\mathbf{x}_i,\underbrace{\cancel{A=a,Y=1}}_{\substack{\text{\makebox[0pt]{Conditional independence given $X$.}} \\ \text{\makebox[0pt]{Similar for other cancellations.}}}})P(Y=1,X=\mathbf{x}_i|A=a) \notag\\
    &\ \ \ \ \ \ \ \ \ + \sum_{i=1}^{N}P(\hat{Y}^{(F)}(X)=0|X=\mathbf{x}_i,\cancel{A=a,Y=0})P(Y=0,X=\mathbf{x}_i|A=a)\notag
\end{align}
\begin{align}
    &= \sum_{i\in \mathcal{I}^+}\overbrace{P(\hat{Y}^{(F)}(X)=1|X=\mathbf{x}_i)}^{\text{\makebox[0pt]{$\mathbf{s}^{(F)}[i]=\mathbf{s}^{*}[i]-(\mathbf{s}^{*}[i]-\mathbf{s}^{(F)}[i])=\mathbf{s}^{*}[i]-|\mathbf{s}^{*}[i]-\mathbf{s}^{(F)}[i]|=\mathbf{s}^{*}[i]-\mathbf{m}[i]$ since $i\in \mathcal{I}^+$.}}}\underbrace{P(Y=1,X=\mathbf{x}_i|A=a)}_{\substack{\text{\makebox[0pt]{$=\mathbf{p}^{1}_a[i]$ by definition.}} \\ \text{\makebox[0pt]{Similar for other substitutions.}}}} \notag
\end{align}
\begin{align}
    &\ \ \ \ \ \ \ \ \ + \sum_{i\in \mathcal{I}^+}\overbrace{P(\hat{Y}^{(F)}(X)=0|X=\mathbf{x}_i)}^{\text{\makebox[0pt]{ $1-\mathbf{s}^{(F)}[i]=1-\mathbf{s}^{*}[i]+(\mathbf{s}^{*}[i]-\mathbf{s}^{(F)}[i])=1-\mathbf{s}^{*}[i]+|\mathbf{s}^{*}[i]-\mathbf{s}^{(F)}[i]|=1-\mathbf{s}^{*}[i]+\mathbf{m}[i]$ since $i\in \mathcal{I}^+$}. }}P(Y=0,X=\mathbf{x}_i|A=a)\notag\\
    &\ \ \ \ \ \ \ \ +\sum_{i\in \mathcal{I}^-}\overbrace{P(\hat{Y}^{(F)}(X)=1|X=\mathbf{x}_i)}^{\text{\makebox[0pt]{$\mathbf{s}^{(F)}[i]=\mathbf{s}^{*}[i]-(\mathbf{s}^{*}[i]-\mathbf{s}^{(F)}[i])=\mathbf{s}^{*}[i]+|\mathbf{s}^{*}[i]-\mathbf{s}^{(F)}[i]|=\mathbf{s}^{*}[i]-\mathbf{m}[i]$ since $i\in \mathcal{I}^-$.}}}P(Y=1,X=\mathbf{x}_i|A=a) \notag\\
    &\ \ \ \ \ \ \ \ \ + \sum_{i\in \mathcal{I}^-}\overbrace{P(\hat{Y}^{(F)}(X)=0|X=\mathbf{x}_i)}^{\text{\makebox[0pt]{$1-\mathbf{s}^{(F)}[i]=1-\mathbf{s}^{*}[i]+(\mathbf{s}^{*}[i]-\mathbf{s}^{(F)}[i])=1-\mathbf{s}^{*}[i]-|\mathbf{s}^{*}[i]-\mathbf{s}^{(F)}[i]|=1-\mathbf{s}^{*}[i]+\mathbf{m}[i]$ since $i\in \mathcal{I}^-$}.}}P(Y=0,X=\mathbf{x}_i|A=a)\notag\\
    & =(\mathbf{s}^{*}-\mathbf{m})^T\mathbf{p}^1_a +(\mathbf{1}_{N}-\mathbf{s}^{*}+\mathbf{m})^T\mathbf{p}^0_a.\notag
\end{align}
Through similar derivation, we obtain $P(\hat{Y}^{(F)}(X)=Y|A=b) = (\mathbf{s}^{*}-\mathbf{m})^T\mathbf{p}^1_b +(\mathbf{1}_{N}-\mathbf{s}^{*}+\mathbf{m})^T\mathbf{p}^0_b$. Hence, 
\begin{align}
    |P(\hat{Y}^{(F)}(X)=Y|A=a)-P(\hat{Y}^{(F)}(X)=Y|A=b)|\leq\epsilon_{EA} \notag
\end{align}
reduces to 
\begin{align}
|(\mathbf{s}^{*}-\mathbf{m})^T(\mathbf{p}^1_a-\mathbf{p}^1_b) +(\mathbf{1}_{N}-\mathbf{s}^{*}+\mathbf{m})^T(\mathbf{p}^0_a-\mathbf{p}^0_b)|\leq\epsilon_{EA} \notag
\end{align}
or
\begin{align}
|(\mathbf{p}_{a}^{1}-\mathbf{p}_{b}^{1})^T(\mathbf{s}^{*}-\mathbf{m})+(\mathbf{p}_{a}^{0}-\mathbf{p}_{b}^{0})^T(\mathbf{1}_{N}-\mathbf{s}^{*}+\mathbf{m})|\leq \epsilon_{EA}. \notag
\end{align}



\section{Derivation of Minimization Problem (5)}
\label{trans_deriv}
Let $S^{(F)}:\mathcal{X}\rightarrow[0,1]$ be the optimal fair randomized scoring function whose output represents the probability of producing a $1$ label for a given feature vector, and let:
\begin{equation}
    \hat{Y}^{(F)}(\mathbf{x})=\begin{cases}
        1&\text{, w.p. } S^{(F)}(\mathbf{x})\\
        0&\text{, w.p. } 1-S^{(F)}(\mathbf{x})
    \end{cases}\notag
\end{equation}
be the optimal fair classifier for the original non-transformed space of feature vectors. In this section, we derive the objective function and fairness constraints used in minimization problem (5). Let $T$ be the random variable produced by applying the decorrelation transformation to $X$. We will first show that the $Accuracy$ term in minimization problem (5) is equal to $P(\hat{Y}^{(F)}(T)=Y)$, which means that minimizing the negative of this term is equivalent to maximizing the fair Bayesian classifier's accuracy on the space of transformed feature vectors. The $Correlation$ term in minimization problem (5) is self-explanatory, and thus no derivation is required. We will derive each of the fairness constraints separately by showing that they are equivalent to the following list of constraints (again, equalized odds is omitted since it is simply a combination or predictive equality and equal opportunity):
\begin{align}
    & |P(\hat{Y}^{(F)}(T)=1|A=a)-P(\hat{Y}^{(F)}(T)=1|A=b)|\leq\epsilon_{DP} \notag \\
    & |P(\hat{Y}^{(F)}(T)=1|A=a,Y=1)-P(\hat{Y}^{(F)}(T)=1|A=b,Y=1)|\leq\epsilon_{EOp} \notag \\
    & |P(\hat{Y}^{(F)}(T)=1|A=a,Y=0)-P(\hat{Y}^{(F)}(T)=1|A=b,Y=0)|\leq\epsilon_{PE} \notag \\
    & |P(\hat{Y}^{(F)}(T)=Y|A=a)-P(\hat{Y}^{(F)}(T)=Y|A=b)|\leq\epsilon_{EA} \notag\\
    &e^{-\theta d^2_{\mathcal{X}}(\mathbf{x}_i,\mathbf{x}_j)}|t^{(F)}[i]-t^{(F)}[j]|\leq \epsilon_{IF}, \ \ \ \ \ \ \ \ \forall i,j \ \ \text{  s.t. } \ \  d_{\mathcal{X}}(\mathbf{x}_i,\mathbf{x}_j)\leq \eta, \notag
\end{align}
where for the individual fairness constraint, $t[i]$ represents the score associated with feature vectors from the original $i^{th}$ VQ cell \textit{after} the decorrelation transformation has been applied. The $Fairness$ constraint in minimization problem $(5)$ simply structures all of these constraints in one block matrix form.
\subsection{Derivation of the $Accuracy$ term}
Observe that:
\begin{align}
    &P(\hat{Y}^{(F)}(T)=Y) = P(\hat{Y}^{(F)}(T)=1,Y=1) + P(\hat{Y}^{(F)}(T)=0,Y=0)\notag\\
    &=\sum_{i=1}^{N}\{\underbrace{P(\hat{Y}^{(F)}(T)=1,Y=1,T=\mathbf{x}_i)}_{(A)} + \underbrace{P(\hat{Y}^{(F)}(T)=0,Y=0,T=\mathbf{x}_i)}_{(B)}\}
\end{align}
The derivations to show that terms $(A)$ and $(B)$ are respectively equal $\mathbf{s}^{(F)T}\mathbf{T}\mathbf{p}^{1}$ and $(\mathbf{1}_{N}-\mathbf{s}^{(F)})^T\mathbf{T}\mathbf{p}^{0}$ are similar. Thus, for brevity, we only provide the former derivation.
\begin{align}
    &P(\hat{Y}^{(F)}(T)=1,Y=1,T=\mathbf{x}_i)\notag\\
    &= \overbrace{P(\hat{Y}^{(F)}(T)=1|T=\mathbf{x}_i,\underbrace{\cancel{Y=1}}_{\text{\makebox[0pt]{Conditional independence given $T$.}}})}^{\mathbf{s}^{(F)}[i] \text{ by definition.}}P(T=\mathbf{x}_i,Y=1) \notag\\
\end{align}
Next, note that
\begin{align}
    &P(T=\mathbf{x}_i,Y=1)=\sum_{k=1}^{N}P(T=\mathbf{x}_i,Y=1, X=\mathbf{x}_k) \notag\\
    &=\sum_{k=1}^{N}\underbrace{P(T=\mathbf{x}_i| X=\mathbf{x}_k,\underbrace{\cancel{Y=1}}_{\text{\makebox[0pt]{Conditional independence given $X$.}}}}_{\mathbf{T}[i,k]})\underbrace{P(Y=1, X=\mathbf{x}_k)}_{\mathbf{p}^1[k]}
    =\sum_{k=1}^{N}\mathbf{T}[i,k]\mathbf{p}^1[k].
\end{align}
Plugging (18) into (17), we have that $(A)$ is given by:
\begin{align}
    P(\hat{Y}^{(F)}(T)=1,Y=1,T=\mathbf{x}_i) = \mathbf{s}^{(F)}[i]\sum_{k=1}^{N}\mathbf{T}[i,k]\mathbf{p}^1[k].
\end{align}
Through a similar derivation, coupled with the fact that $P(\hat{Y}^{(F)}(T)=0|T=\mathbf{x}_i,Y=0) = 1-\mathbf{s}^{(F)}[i]$, we can find that $(B)$ is given by:
\begin{align}
    P(\hat{Y}^{(F)}(T)=0,Y=0,T=\mathbf{x}_i) = (1-\mathbf{s}^{(F)}[i])\sum_{k=1}^{N}\mathbf{T}[i,k]\mathbf{p}^0[k].
\end{align}
Finally, substituting $(19)$ and $(20)$ into $(16)$, we obtain that
\begin{align}
    &P(\hat{Y}^{(F)}(T)=Y) =\sum_{i=1}^{N} \{\mathbf{s}^{(F)}[i]\sum_{k=1}^{N}\mathbf{T}[i,k]\mathbf{p}^1[k] +  (1-\mathbf{s}^{(F)}[i])\sum_{k=1}^{N}\mathbf{T}[i,k]\mathbf{p}^0[k]\}\notag\\
    &=\sum_{i=1}^{N}\sum_{k=1}^{N}\mathbf{T}[i,k]\mathbf{s}^{(F)}[i]\mathbf{p}^1[k] + \sum_{i=1}^{N}\sum_{k=1}^{N}\mathbf{T}[i,k](1-\mathbf{s}^{(F)}[i])\mathbf{p}^0[k]\notag\\
    &=  \mathbf{s}^{(F)T}\mathbf{T}\mathbf{p}^{1} +(\mathbf{1}_{N}-\mathbf{s}^{(F)})^T\mathbf{T}\mathbf{p}^{0}=Acc^{(d)}.\notag
\end{align}

\subsection{Demographic Parity:} Observe that:
\begin{align}
    &P(\hat{Y}^{(F)}(T)=1|A=a) = \sum_{i=1}^{N}P(\hat{Y}^{(F)}(T)=1,T=\mathbf{x}_i|A=a) \notag\\
    &= \sum_{i=1}^{N}\overbrace{P(\hat{Y}^{(F)}(T)=1|T=\mathbf{x}_i,\underbrace{\cancel{A=a}}_{\text{\makebox[0pt]{Conditional independence given $T$.}}})}^{\mathbf{s}^{(F)}[i] \text{ by definition.}}P(T=\mathbf{x}_i|A=a)
\end{align}
Next, note that:
\begin{align}
    &P(T=\mathbf{x}_i|A=a) = \sum_{k=1}^{N}P(T=\mathbf{x}_i, X=\mathbf{x}_k|A=a)\notag
\end{align}
\begin{align}
    &=\sum_{k=1}^{N}P(T=\mathbf{x}_i|X=\mathbf{x}_k,\underbrace{\cancel{A=a}}_{\text{\makebox[0pt]{Conditional independence given $X$.}}})P(X=\mathbf{x}_k|A=a)=\sum_{k=1}^{N}\mathbf{T}[i,k]\mathbf{p}_a[k].
\end{align}
Plugging (22) into (21), we obtain:
\begin{align}
    &P(\hat{Y}^{(F)}(T)=1|A=a) = \sum_{k=1}^{N} \mathbf{s}^{(F)}[i]\sum_{k=1}^{N}\mathbf{T}[i,k]\mathbf{p}_a[k] = \sum_{k=1}^{N} \sum_{k=1}^{N}\mathbf{T}[i,k]\mathbf{s}^{(F)}[i]\mathbf{p}_a[k] = \mathbf{p}_a^T\mathbf{T}\mathbf{s}^{(F)}. \notag
\end{align}
A similar derivation yields that $P(\hat{Y}^{(F)}(T)=1|A=b) = \mathbf{p}_b^T\mathbf{T}\mathbf{s}^{(F)}$. Thus, 
\begin{align}
    |P(\hat{Y}^{(F)}(T)=1|A=a)-P(\hat{Y}^{(F)}(T)=1|A=b)|\leq\epsilon_{DP} \notag
\end{align}
is equivalent to 
\begin{align}
    |(\mathbf{p}_a-\mathbf{p}_b)^T\mathbf{T}\mathbf{s}^{(F)}|\leq \epsilon_{DP}. \notag
\end{align}

\subsection{Equal Opportunity:} Observe that:
\begin{align}
    &P(\hat{Y}^{(F)}(T)=1|A=a,Y=1) = \sum_{i=1}^{N}P(\hat{Y}^{(F)}(T)=1,T=\mathbf{x}_i|A=a,Y=1) \notag\\
    &= \sum_{i=1}^{N}\overbrace{P(\hat{Y}^{(F)}(T)=1|T=\mathbf{x}_i,\underbrace{\cancel{A=a,Y=1}}_{\text{\makebox[0pt]{conditional independence}}})}^{\mathbf{s}^{(F)}[i] \text{ by definition.}}P(T=\mathbf{x}_i|A=a,Y=1).
\end{align}
Next, note that:
\begin{align}
    &P(T=\mathbf{x}_i|A=a,Y=1) = \sum_{k=1}^{N}P(T=\mathbf{x}_i, X=\mathbf{x}_k|A=a)\notag\\
    &=\sum_{k=1}^{N}P(T=\mathbf{x}_i|X=\mathbf{x}_k,\underbrace{\cancel{A=a,Y=1}}_{\text{\makebox[0pt]{conditional independence}}})P(X=\mathbf{x}_k|A=a,Y=1)=\sum_{k=1}^{N}\mathbf{T}[i,k]\mathbf{p}_{a,1}[k].
\end{align}

Plugging (24) into (23), we obtain:
\begin{align}
    &P(\hat{Y}^{(F)}(T)=1|A=a) = \sum_{k=1}^{N} \mathbf{s}^{(F)}[i]\sum_{k=1}^{N}\mathbf{T}[i,k]\mathbf{p}_{a,1}[k] = \sum_{k=1}^{N} \sum_{k=1}^{N}\mathbf{T}[i,k]\mathbf{s}^{(F)}[i]\mathbf{p}_{a,1}[k] = \mathbf{p}_{a,1}^T\mathbf{T}\mathbf{s}^{(F)}. \notag
\end{align}
Through similar derivation, we can find that $P(\hat{Y}^{(F)}(T)=1|A=b,Y=1) = \mathbf{p}_{b,1}^T\mathbf{T}\mathbf{s}^{(F)}$. Thus, 
\begin{align}
    |P(\hat{Y}^{(F)}(T)=1|A=a,Y=1)-P(\hat{Y}^{(F)}(T)=1|A=b,Y=1)|\leq\epsilon_{EOp} \notag
\end{align}
is equivalent to 
\begin{align}
    |(\mathbf{p}_{a,1}-\mathbf{p}_{b,1})^T\mathbf{T}\mathbf{s}^{(F)}|\leq \epsilon_{EOp}. \notag
\end{align}

\subsection{Predictive Equality}
Observing that $P(\hat{Y}^{(F)}(T)=0|T=\mathbf{x}_i) = 1-\mathbf{s}^{(F)}[i],$ we simply replace this in the Equal Opportunity derivation to obtain:
\begin{align}
    |(\mathbf{p}_{a,0}-\mathbf{p}_{b,0})^T\mathbf{T}(\mathbf{1}_{N}-\mathbf{s}^{(F)})|\leq \epsilon_{EOp}. \notag
\end{align}

\subsection{Equal Accuracy:} Observe that:
\begin{align}
    &P(\hat{Y}^{(F)}(T)=Y|A=a) = P(\hat{Y}^{(F)}(T)=1,Y=1|A=a) + P(\hat{Y}^{(F)}(T)=0,Y=0|A=a)\notag\\
    &=\sum_{i=1}^{N}\underbrace{P(\hat{Y}^{(F)}(T)=1,Y=1,T=\mathbf{x}_i|A=a)}_{(A)} + \underbrace{P(\hat{Y}^{(F)}(T)=0,Y=0,T=\mathbf{x}_i|A=a)}_{(B)}
\end{align}
The derivations for $(A)$ and $(B)$ are similar, so we focus on deriving $(A)$ here.
\begin{align}
    & P(\hat{Y}^{(F)}(T)=1,Y=1,T=\mathbf{x}_i|A=a) =  \notag \\
    &=\overbrace{P(\hat{Y}^{(F)}(T)=1|T=\mathbf{x}_i,\underbrace{\cancel{A=a,Y=1}}_{\text{\makebox[0pt]{conditional independence}}})}^{\mathbf{s}^{(F)}[i] \text{ by definition.}}P(T=\mathbf{x}_i,Y=1|A=a)
\end{align}
Next, note that
\begin{align}
    &P(T=\mathbf{x}_i,Y=1|A=a)=\sum_{k=1}^{N}P(T=\mathbf{x}_i,Y=1, X=\mathbf{x}_k|A=a) \notag\\
    &=\sum_{k=1}^{N}\underbrace{P(T=\mathbf{x}_i| X=\mathbf{x}_k,\underbrace{\cancel{Y=1,A=a}}_{\text{\makebox[0pt]{conditional independence}}}}_{\mathbf{T}[i,k]})\underbrace{P(Y=1, X=\mathbf{x}_k|A=a)}_{\mathbf{p}_a^1[k]}
    =\sum_{k=1}^{N}\mathbf{T}[i,k]\mathbf{p}_a^1[k].
\end{align}
Plugging (27) into (26), we have that $(A)$ is given by:
\begin{align}
    P(\hat{Y}^{(F)}(T)=1,Y=1,T=\mathbf{x}_i|A=a) = \mathbf{s}^{(F)}[i]\sum_{k=1}^{N}\mathbf{T}[i,k]\mathbf{p}_a^1[k].
\end{align}
Through a similar derivation, coupled with the fact that $P(\hat{Y}^{(F)}(T)=0|T=\mathbf{x}_i,A=a,Y=0) = 1-\mathbf{s}^{(F)}[i]$, we can find that $(B)$ is given by:
\begin{align}
    P(\hat{Y}^{(F)}(T)=0,Y=0,T=\mathbf{x}_i|A=a) = (1-\mathbf{s}^{(F)}[i])\sum_{k=1}^{N}\mathbf{T}[i,k]\mathbf{p}_a^0[k].
\end{align}
Finally, substituting $(28)$ and $(29)$ into $(25)$, we obtain that
\begin{align}
    &P(\hat{Y}^{(F)}(T)=Y|A=a) =\sum_{i=1}^{N} \mathbf{s}^{(F)}[i]\sum_{k=1}^{N}\mathbf{T}[i,k]\mathbf{p}_a^1[k] +  (1-\mathbf{s}^{(F)}[i])\sum_{k=1}^{N}\mathbf{T}[i,k]\mathbf{p}_a^0[k]\notag\\
    &=\sum_{i=1}^{N}\sum_{k=1}^{N}\mathbf{T}[i,k]\mathbf{s}^{(F)}[i]\mathbf{p}_a^1[k] + \sum_{i=1}^{N}\sum_{k=1}^{N}\mathbf{T}[i,k](1-\mathbf{s}^{(F)}[i])\mathbf{p}_a^0[k]\notag\\
    &=  \mathbf{p}_{a}^{1T}\mathbf{T}\mathbf{s}^{(F)} +  \mathbf{p}_{a}^{0T}\mathbf{T}(\mathbf{1}_{N}-\mathbf{s}^{(F)}).\notag
\end{align}
Through similar derivation, we obtain $P(\hat{Y}^{(F)}(T)=Y|A=b)=\mathbf{p}_{b}^{1T}\mathbf{T}\mathbf{s}^{(F)} +  \mathbf{p}_{b}^{0T}\mathbf{T}(\mathbf{1}_{N}-\mathbf{s}^{(F)}).$ Thus,
\begin{align}
    |P(\hat{Y}^{(F)}(T)=Y|A=a)-P(\hat{Y}^{(F)}(T)=Y|A=b)|\leq\epsilon_{EA}\notag
\end{align}
is equivalent to 
\begin{align}
    |(\mathbf{p}_{a}^{1}-\mathbf{p}_{b}^{1})^T\mathbf{T}\mathbf{s}^{(F)} +  (\mathbf{p}_{a}^{0}-\mathbf{p}_{b}^{0})^T\mathbf{T}(\mathbf{1}_{N}-\mathbf{s}^{(F)})|\leq\epsilon_{EA}\notag
\end{align}
\subsection{Local Individual Fairness:} Let $i$ and $j$ represent the arbitrary $n^{th}$ pair of cell centroids that satisfy $d_{\mathcal{X}}(\mathbf{x}_i,\mathbf{x}_j)\leq \eta$. Observe that the probability that a feature vector, $\mathbf{x}_i$, is classified as $1$ by $S^{(F)}$ after it has been transformed by $T$ is given by:
\begin{align}
    t[i] = E[S^{(F)}(T)|X=\mathbf{x}_i]=\sum_{k=1}^{N}S^{(F)}(\mathbf{x}_k)P(T=\mathbf{x}_k|X=\mathbf{x}_i)
    =\sum_{k=1}^{N}\mathbf{s}^{(F)}[k]\mathbf{T}[k,i].\notag
\end{align}
Similarly, $t[j]=\sum_{k=1}^{N}\mathbf{s}^{(F)}[k]\mathbf{T}[k,j].$ Now, let $\mathbf{w}_{n}$ be a vector in which every entry is zero, except for the $i^{th}$ and $j^{th}$ entries, which contain values of $e^{-\theta d^2_{\mathcal{X}}(\mathbf{x}_i,\mathbf{x}_j)}$ and $-e^{-\theta d^2_{\mathcal{X}}(\mathbf{x}_i,\mathbf{x}_j)}$, respectively. Then,
\begin{align}
    &e^{-\theta d^2_{\mathcal{X}}(\mathbf{x}_i,\mathbf{x}_j)}\Big|t^{(F)}[i]-t^{(F)}[j]\Big| = e^{-\theta d^2_{\mathcal{X}}(\mathbf{x}_i,\mathbf{x}_j)}\Big|\sum_{k=1}^{N}\mathbf{s}^{(F)}[k](\mathbf{T}[k,i]-\mathbf{T}[k,j])\Big| \notag\\
    &\Big|\sum_{k=1}^{N}\Big(\mathbf{s}^{(F)}[k](\mathbf{T}[k,i]e^{-\theta d^2_{\mathcal{X}}(\mathbf{x}_i,\mathbf{x}_j)}-\mathbf{T}[k,j]e^{-\theta d^2_{\mathcal{X}}(\mathbf{x}_i,\mathbf{x}_j)}\Big)\Big|=\Big|\sum_{l=1}^{N}\sum_{k=1}^{N}\mathbf{s}^{(F)}[k]\mathbf{T}[k,i]\mathbf{w}_{n}[l]\Big| \notag\\
    &=\Big|\mathbf{w}_n^T\mathbf{T}\mathbf{s}^{(F)}|.\notag
\end{align}
Thus,
\begin{align}
    e^{-\theta d_{\mathcal{X}}(\mathbf{x}_i,\mathbf{x}_j)}|t^{(F)}[i]-t^{(F)}[j]|\leq \epsilon_{IF}\notag
\end{align}
is equivalent to:
\begin{align}
    \Big|\mathbf{w}_n^T\mathbf{T}\mathbf{s}^{(F)}|\leq \epsilon_{IF}.\notag
\end{align}
To enforce this constraint for all $i$ and $j$, such that $d_{\mathcal{X}}(\mathbf{x}_i,\mathbf{x}_j)\leq \eta$, simply stack all such $\mathbf{w}_n^T$ into a matrix $\mathbf{W}$. The associated constraint is then given by:
\begin{align}
    \Big|\mathbf{W}\mathbf{T}\mathbf{s}^{(F)}|\leq \mathbf{1}_{N}\epsilon_{IF}.\notag
\end{align}



\section{Proof of Claim 1}
\label{claim1}


We will now show that the following minimization problem is convex:
\begin{align}
\max_{\bm{\rho}}\min_{\mathbf{T}\in\mathcal{P}} 
    &-\lambda(\mathbf{s}^{(F)T}\mathbf{T}\mathbf{p}^{1} + (\mathbf{1}_{N}-\mathbf{s}^{(F)})^T\mathbf{T}\mathbf{p}^{0}) + \beta\|\mathbf{T}(\mathbf{p}_a-\mathbf{p}_b)\|_1 \notag \\
    & + \langle\bm{\rho},\max(\tilde{f}(\mathbf{T})-\mathbf{\tilde{f}}, \mathbf{0}_{2B+8})\rangle + \frac{\tau}{2}\|\max(\tilde{f}(\mathbf{T})-\mathbf{\tilde{f}}, \mathbf{0}_{2B+8})\|_2^2.\notag
\end{align}
\begin{proof}
Observing that the sum of convex functions is convex, it suffices to prove that this minimization problem is convex by showing that every term in the problem is convex in $\mathbf{T}$, which we will now do. $-\lambda(\mathbf{s}^{(F)T}\mathbf{T}\mathbf{p}^{1} + (\mathbf{1}_{N}-\mathbf{s}^{(F)})^T\mathbf{T}\mathbf{p}^{0})$ is linear in $\mathbf{T}$, and thus convex. $\beta\|\mathbf{T}(\mathbf{p}_a-\mathbf{p}_b)\|_1$ is convex since the composition of a convex function with an affine function is convex. 

To show that the final terms in the minimization problem are convex, we begin by showing that $g(\mathbf{T})=\max(\tilde{f}(\mathbf{T})-\mathbf{\tilde{f}}, \mathbf{0}_{2B+8})$ is elementwise convex using the definition of convexity. Specifically, for $\theta\in[0,1]$ and $\mathbf{T}_1,\mathbf{T}_2\in\mathcal{P}$,  observe that:
\begin{align}
    &g(\theta \mathbf{T}_1 + (1-\theta)\mathbf{T}_2 ) = \max(\tilde{f}(\theta \mathbf{T}_1 + (1-\theta)\mathbf{T}_2)-\mathbf{\tilde{f}}, \ \ \mathbf{0}_{2B+8}) \notag\\
    &=\max(\underbrace{\mathbf{\tilde{P}}}_{\begin{bmatrix}-\mathbf{P} \\ \mathbf{P}\end{bmatrix}}\bigl[\mathbf{M}\circ\bigl( \mathbf{\tilde{I}} (\theta\mathbf{T}_1 + (1-\theta)\mathbf{T}_2)\mathbf{\tilde{I}}^T \bigr) \bigr]\mathbf{s}^{(F)}-\mathbf{\tilde{f}}, \ \ \mathbf{0}_{2B+8}) \notag\\
    &=\max(\theta\underbrace{\{\mathbf{\tilde{P}}\bigl[\mathbf{M}\circ\bigl( \mathbf{\tilde{I}} \mathbf{T}_1 \mathbf{\tilde{I}}^T \bigr) \bigr]\mathbf{s}^{(F)}-\mathbf{\tilde{f}}\}}_{\tilde{f}(\mathbf{T}_1)} + (1-\theta)\underbrace{\{\mathbf{\tilde{P}}\bigl[\mathbf{M}\circ\bigl( \mathbf{\tilde{I}} \mathbf{T}_2 \mathbf{\tilde{I}}^T \bigr) \bigr]\mathbf{s}^{(F)}-\mathbf{\tilde{f}}\} }_{\tilde{f}(\mathbf{T}_2)}, \ \ \mathbf{0}_{2B+8}) \notag\\
    &=\max(\theta\tilde{f}(\mathbf{T}_1) + (1-\theta)\tilde{f}(\mathbf{T}_2), \ \ \mathbf{0}_{2B+8}) \notag\\
    &\underbrace{\leq}_{\substack{\text{\makebox[0pt]{Elementwise by convexity}} \\ \text{\makebox[0pt]{of pointwise max.}}}} \max(\theta\tilde{f}(\mathbf{T}_1), \ \ \mathbf{0}_{2B+8}) + \max((1-\theta)\tilde{f}(\mathbf{T}_2), \ \ \mathbf{0}_{2B+8}) \notag\\
    &=\theta\max(\tilde{f}(\mathbf{T}_1), \ \ \mathbf{0}_{2B+8}) + (1-\theta)\max(\tilde{f}(\mathbf{T}_2), \ \ \mathbf{0}_{2B+8}) \notag\\
    &=\theta g(\mathbf{T}_1) + (1-\theta)g(\mathbf{T}_2) \notag
\end{align}
Thus, the inner product between $\bm{\rho}$ and $g(\mathbf{T})$ is simply a linear combination of convex functions, making the second to last term in the optimization problem convex. Finally, the convexity of the last term in the optimization problem follows from the fact that $g(\mathbf{T})$ is a non-negative convex function and the norm is convex and non-decreasing over the set $\mathbb{R}_{\geq 0}$.
\end{proof}

\section{Fairness-Accuracy Trade-off (Sensitive Aware)}
\label{sec_fair_aware}

In Section \ref{fairness}, the minimization problem for analyzing the trade-off between accuracy and fairness was formulated under the assumption that the sensitive attribute is unavailable. In this section, we adapt this formulation for the situation in which the sensitive attribute is allowed to be used by the unconstrained and fair Bayesian classifiers to make its decisions. With this information, the solution of the Bayesian classifier is given by $\mathbf{s}^{*}=\begin{bmatrix} \mathbf{s}_a^{*T}, & \mathbf{s}_b^{*T}\end{bmatrix}^T$, where
\begin{equation}
    \mathbf{s}_a^{*}[i] = \argmax_{y}\mathbf{p}_a^{y}[i] \ \ \ \ \ \ \ \ \ \ \text{and} \ \ \ \ \ \ \ \ \mathbf{s}_b^{*}[i] = \argmax_{y}\mathbf{p}_b^{y}[i], \forall i,\notag
\end{equation}
which again produces binary decisions. The accuracy of this classifier is given by $Acc^*=\sum_{g\in\mathcal{A}}\sum_{i=1}^{N}\mathbf{p}^{\mathbf{s}_g^{*}[i]}[i]$. Constructing the minimization problem for finding the  optimal fair Bayesian classifier's decisions is extended from the minimization problem presented in Section \ref{fairness} by separating its decisions by group. Towards this end, we let $\mathbf{m}_g = \mathbf{s}_g^{*}-\mathbf{s}_g^{(F)}, g\in\mathcal{A},$ and $\mathcal{I}^+_g$ and $\mathcal{I}^-_g$, be the set of indices associated with the positive and negative elements of $\mathbf{p}^{1,g}-\mathbf{p}^{0,g}$, respectively. Then, the minimization problem is given by:
\begin{align}
    &\min_{\{\mathbf{m}_g\}}\sum_{g\in\mathcal{A}}(\mathbf{p}^{1,g}-\mathbf{p}^{0,g})^T\mathbf{m}_g, \notag\\
    & |\mathbf{p}_{a}^T(\mathbf{s}_a^{*}-\mathbf{m}_a)-\mathbf{p}_b^T(\mathbf{s}_b^{*}-\mathbf{m}_b)|\leq \epsilon_{DP} \ \ \ \ \ \ \ \ \ \ \ \ \ \ \ \ \ \ \ \ \ \ \ \ \ \ \ \ \ \ \ (DP)\notag\\
    &|\mathbf{p}_{a,0}^T(\mathbf{s}_a^{*}-\mathbf{m}_a)-\mathbf{p}_{b,0}^T(\mathbf{s}_b^{*}-\mathbf{m}_b)|\leq \epsilon_{PE} \ \ \ \ \ \ \ \ \ \ \ \ \ \ \ \ \ \ \ \ \ \ \ \ \ \ \ \ (PE) \notag\\
    &|\mathbf{p}_{a,1}^T(\mathbf{s}_a^{*}-\mathbf{m}_a)-\mathbf{p}_{b,1}^T(\mathbf{s}_b^{*}-\mathbf{m}_b)|\leq \epsilon_{EOp} \ \ \ \ \ \ \ \ \ \ \ \ \ \ \ \ \ \ \ \ \ \ \ \ \ \ (EOp)\notag\\
    &\epsilon_{EOp} = \epsilon_{PE} \ \ \ \ \ \ \ \ \ \ \ \ \ \ \ \ \ \ \ \ \ \ \ \ \ \ \ \ \ \ \ \ \ \ \ \ \ \ \ \ \ \ \ \ \ \ \ \ \ \ \ \ \ \ \ \ \ \ \ \ \ \ \ \ \ \ \ \ \ \ \ \ \ (EOd)\notag\\
    &|\mathbf{p}_{a}^{1T}(\mathbf{s}_a^{*}-\mathbf{m}_a)-\mathbf{p}_{b}^{1T}(\mathbf{s}_b^{*}-\mathbf{m}_b)+ \notag\\
    &\mathbf{p}_{a}^{0T}(\mathbf{1}_{N}-\mathbf{s}_a^{*}+\mathbf{m}_a)-\mathbf{p}_{b}^{0T}(\mathbf{1}_{N}-\mathbf{s}_b^{*}+\mathbf{m}_b)|\leq \epsilon_{EA} \ \ \ \ \ \ \ \ \ \ \  (EA)\notag\\
    & |\mathbf{W}\sum_{g\in\mathcal{A}}\text{diag}(\mathbf{p}_g)(\mathbf{s}_g^{*}-\mathbf{m}_g)|\leq \epsilon_{IF}\mathbf{1}_{N}, \ \ \ \ \ \ \ \ \ \ \ \ \ \ \ \ \ \ \ \ \ \ \ \ \ \ \ \ \ \  (Ind) \notag \\
    &0\leq \mathbf{m}_g[i]\leq 1,  \ \ \ \ \ i \in \mathcal{I}^+_g, \ \ \ \ \ g\in\mathcal{A}\notag\\
    -&1\leq \mathbf{m}_g[i]\leq 0, \ \ \ \ i \in \mathcal{I}^-_g, \ \ \ \ \ g\in\mathcal{A},
\end{align}
where $\mathbf{W}$ is defined as in equation (3). The derivation of each of these terms is similar to the derivation provided in Appendix \ref{app_LP_const}, with the added step of expanding each term using group affiliation information.

\section{Transfer Fairness to Decorrelated Domain (Sensitive Aware)}
\label{sec_priv_aware}

In Section \ref{fair_priv} the minimization problem for transferring fairness to a decorrelated domain was formulated under the assumption of unawareness of the sensitive attribute. Availability of the sensitive attribute provides more flexibility in construction of the mapping, allowing us to transform the feature vectors belonging to each group using group-specific transformations. Specifically, given a fixed fair Bayesian Classifier's decision map, $\mathbf{s}^{(F)} = \begin{bmatrix} \mathbf{s}_a^{(F)T}, & \mathbf{s}_b^{(F)T} \end{bmatrix}^T,$ where the first $N$ entries represent the scores associated with Group $a$ and the remaining $N$ entries represent the scores for Group $b$, we aim to optimize for two matrices, $\mathbf{T}_a,\mathbf{T}_b\in\mathcal{P},$ using the following optimization problem:
\begin{align}
    \min_{\mathbf{T}_a,\mathbf{T}_b\in\mathcal{P}} 
    &-\lambda\sum_{g\in\mathcal{A}}(\mathbf{s}^{(F)T}\mathbf{T}_g\mathbf{p}^{1,g} + (\mathbf{1}_{N}-\mathbf{s}^{(F)})^T\mathbf{T}_g\mathbf{p}^{0,g}) \notag\\
    &+ \beta\|\mathbf{T}_a\mathbf{p}_a-\mathbf{T_b}\mathbf{p}_b\|_1 \notag\\
    &  \text{s.t. } \ \ \ |f(\mathbf{T}_a,\mathbf{T}_b)|\leq \mathbf{F} \ \ \ \  \textrm{\it (Fairness condition)}.
\end{align}
The derivation of each of these terms is similar to the derivation provided in Appendix \ref{trans_deriv}, with the added step of expanding each term using group affiliation information. The terms in the objective function and the $Fairness$ constraint for this optimization problem still have the same interpretation as in the situation of unawareness of the sensitive attribute. Note that in this case, the row dimension of the matrices, $\mathbf{T}_a$ and $\mathbf{T}_b$, are twice that of the matrix, $\mathbf{T}$, from minimization problem (5).
This is because, each VQ cell effectively has two bins of information associated with it (one for each group in the original space of feature vectors). Nevertheless, optimization over these matrices ensures that in the decorrelated domain, all VQ cells are conditionally independent of the sensitive attribute. 
The function $f(\mathbf{T}_a,\mathbf{T}_b)\in\mathbb{R}^{1\times(B+4)}$ is now given by the following equation:
\begin{align}
    f(\mathbf{T}_a,\mathbf{T}_b) &= \underbrace{\begin{bmatrix} \mathbf{s}^{(F)} & (\mathbf{1}_{N}-\mathbf{s}^{(F)})
    \end{bmatrix} }_{\mathbf{\tilde{s}}^{(F)T}}\Biggl(
    \underbrace{\begin{bmatrix}\mathbf{T}_a & \mathbf{0}_{2N,N}\\
    \mathbf{0}_{2N,N} & \mathbf{T}_a
    \end{bmatrix}}_{\mathbf{\tilde{T}}_a} \underbrace{\begin{bmatrix}\mathbf{p}_a & \mathbf{p}_{a,1} & \mathbf{0}_{N} & \mathbf{p}_a^{1} & \mathbf{W}_a^T \\ \mathbf{0}_{N} & \mathbf{0}_{N} & \mathbf{p}_{a,0} & \mathbf{p}_a^{0} & \mathbf{0}_{B,N}
    \end{bmatrix}}_{\mathbf{P}_a} \notag\\
    &-\underbrace{\begin{bmatrix}\mathbf{T}_b & \mathbf{0}_{2N,N}\\
    \mathbf{0}_{2N,N} & \mathbf{T}_b
    \end{bmatrix}}_{\mathbf{\tilde{T}}_b} \underbrace{\begin{bmatrix}\mathbf{p}_b & \mathbf{p}_{b,1} & \mathbf{0}_{N} & \mathbf{p}_b^{1} & -\mathbf{W}_b^T \\ \mathbf{0}_{N} & \mathbf{0}_{N} & \mathbf{p}_{b,0} & \mathbf{p}_b^{0} & \mathbf{0}_{B,N}
    \end{bmatrix}}_{\mathbf{P}_b} \Biggr)\notag\\
    &=\mathbf{\tilde{s}}^{(F)T}(\mathbf{\tilde{T}}_a\mathbf{P}_a-\mathbf{\tilde{T}}_b\mathbf{P}_b) \notag
\end{align}
where, for $g\in\mathcal{A}$, $\mathbf{\tilde{T}}_g$ can be directly written as a function of $\mathbf{T}_g$:
\begin{equation}
    \mathbf{\tilde{T}}_g = \underbrace{\begin{bmatrix}\mathbf{1}_{2N, N} &  \mathbf{0}_{2N, N}\\ \mathbf{0}_{2N, N} & \mathbf{1}_{2N, N}\end{bmatrix}}_{\tilde{\mathbf{M}}}\circ \Biggl( \underbrace{\begin{bmatrix} \mathbf{I}_{2N,2N} \\ \mathbf{I}_{2N,2N} \end{bmatrix}}_{\mathbf{\tilde{I}}_2} \mathbf{T}_g \underbrace{\begin{bmatrix} \mathbf{I}_{N,N} & \mathbf{I}_{N,N}\end{bmatrix}}_{\mathbf{\tilde{I}}_1^{T}}\Biggl). \notag
\end{equation}
Still, each of the first four elements of $|f(\mathbf{T}_a,\mathbf{T}_b)|$ captures the degree to which a particular group fairness notion is violated in the transformed space, while the remaining $B$ terms enforce individual fairness. Setting $\mathbf{F}=[\epsilon_{DP}, \ \epsilon_{PE}, \ \epsilon_{EOp}, \ \epsilon_{EOd}, \ \mathbf{1}_{B}^T\epsilon_{IF}]$, preserves the group and individual fairness constraints. 

Again, we can reformulate the $Fairness$ constraint as an equality constraint of the form:
$\max(\tilde{f}(\mathbf{T}_a,\mathbf{T}_b)-\mathbf{\tilde{F}}, \mathbf{0}_{1,2B+8})=\mathbf{0}_{1,2B+8}$, where 
\begin{align}
\tilde{f}(\mathbf{T}_a,\mathbf{T}_b) = \mathbf{\tilde{s}}^{(F)T}(\mathbf{\tilde{T}_a}\underbrace{\begin{bmatrix}\mathbf{P}_a & -\mathbf{P}_a\end{bmatrix}}_{\tilde{\mathbf{P}}_a}-\mathbf{\tilde{T}_b}\underbrace{\begin{bmatrix}\mathbf{P}_b & -\mathbf{P}_b\end{bmatrix}}_{\tilde{\mathbf{P}}_b})
\ \ \ \ \ \ \ 
\text{and}
\ \ \ \ \ \ \ 
\mathbf{\tilde{F}} = \begin{bmatrix}\mathbf{F} & \mathbf{F}\end{bmatrix} \notag
\end{align}

The final Augmented Lagrangian can be formed as:
\begin{align}
\max_{\bm{\rho}}\min_{\mathbf{T}_a,\mathbf{T}_b\in\mathcal{P}} 
    &-\lambda\underbrace{\sum_{g\in\mathcal{A}}(\mathbf{s}^{(F)T}\mathbf{T}_g\mathbf{p}^{1,g} + (\mathbf{1}_{N}-\mathbf{s}^{(F)})^T\mathbf{T}_g\mathbf{p}^{0,g})}_{Accuracy} + \beta\underbrace{\|\mathbf{T}_a\mathbf{p}_a-\mathbf{T_b}\mathbf{p}_b\|_1}_{Correlation} \notag\\
    & + \underbrace{\langle\bm{\rho},\max(\tilde{f}(\mathbf{T}_a,\mathbf{T}_b)-\mathbf{\tilde{F}}, \mathbf{0}_{1,2B+8})\rangle}_{Augmentation_1} \notag\\
    &+ \frac{\tau}{2}\underbrace{\|\max(\tilde{f}(\mathbf{T}_a,\mathbf{T}_b)-\mathbf{\tilde{F}}, \mathbf{0}_{1,2B+8})\|_{F}^2}_{Augmentation_2}
\end{align}
Minimization problem (32) is convex according to Claim \ref{claim2}. Solving this minimization problem is equivalent to solving minimization problem (31) and can be done by applying the alternating direction method of multipliers (ADMM) to optimize for $\mathbf{T}_a$ and $\mathbf{T}_b$ until convergence using the following updates.
\begin{align}
&\mathbf{T}_a^{k+1}=\argmin_{\mathbf{T}_a\in\mathcal{P}}L(\mathbf{T}_a,\mathbf{T}_b^{k}, \bm{\rho}^{k})\notag\\
&\mathbf{T}_b^{k+1}=\argmin_{\mathbf{T}_b\in\mathcal{P}}L(\mathbf{T}_a^{k+1},\mathbf{T}_b,\bm{\rho}^{k})\notag\\
&\bm{\rho}^{k+1}=\bm{\rho}^{k}+\tau \max(\tilde{f}(\mathbf{T}_a^{k+1},\mathbf{T}_b^{k+1})-\mathbf{\tilde{F}}, \mathbf{0}_{1,2B+8}).\notag
\end{align}

\begin{claim}
\label{claim2} Minimization problem (32) is jointly convex in $\mathbf{T}_a$ and $\mathbf{T}_b$
\end{claim}
\begin{proof}
    Observing that the sum of jointly convex functions is jointly convex, it suffices to prove that this minimization problem is jointly convex if every term in the problem is jointly convex in $\mathbf{T}_a$ and $\mathbf{T}_b$. Thus, we will show that each term in the problem is jointly convex in $\mathbf{T}_a$ and $\mathbf{T}_b$. Consider $\mathbf{T}_{1,a},\mathbf{T}_{2,a},\mathbf{T}_{1,b},\mathbf{T}_{2,b}\in\mathcal{P}$ and $\theta \in [0,1]$.

    \underline{Joint convexity of the $Accuracy$ term:}\\
    Let $f(\mathbf{T}_a,\mathbf{T}_b)$ represent the $Accuracy$ term in minimization problem $(32)$.
    \begin{align}
        &f(\theta\mathbf{T}_{1,a}+(1-\theta)\mathbf{T}_{2,a},\theta\mathbf{T}_{1,b}+(1-\theta)\mathbf{T}_{2,b}) \notag\\
        &=\sum_{g\in\mathcal{A}}(\mathbf{s}^{(F)T}(\theta\mathbf{T}_{1,g}+(1-\theta)\mathbf{T}_{2,g})\mathbf{p}^{1,g} + (\mathbf{1}_{N}-\mathbf{s}^{(F)})^T(\theta\mathbf{T}_{1,g}+(1-\theta)\mathbf{T}_{2,g})\mathbf{p}^{0,g})\notag\\
        &=\theta\sum_{g\in\mathcal{A}}\{\mathbf{s}^{(F)T}\mathbf{T}_{1,g}\mathbf{p}^{1,g} + (\mathbf{1}_{N}-\mathbf{s}^{(F)})^T\mathbf{T}_{1,g}\mathbf{p}^{0,g}\} \notag\\
        & \ \ \ \ \ \ \ \ \ \ \ \ \ \ \ + (1-\theta)\sum_{g\in\mathcal{A}}\{\mathbf{s}^{(F)T}\mathbf{T}_{2,g}\mathbf{p}^{2,g} + (\mathbf{1}_{N}-\mathbf{s}^{(F)})^T\mathbf{T}_{2,g}\mathbf{p}^{0,g}\}\notag\\
        &= \theta f(\mathbf{T}_{1,a},\mathbf{T}_{1,b}) + (1-\theta) f(\mathbf{T}_{2,a},\mathbf{T}_{2,b})\notag
    \end{align}
    
    \underline{Joint convexity of the $Correlation$ term:}\\
    Let $f(\mathbf{T}_a,\mathbf{T}_b)$ represent the $Correlation$ term in minimization problem $(32)$.
    \begin{align}
        &f(\theta\mathbf{T}_{1,a}+(1-\theta)\mathbf{T}_{2,a},\theta\mathbf{T}_{1,b}+(1-\theta)\mathbf{T}_{2,b})\notag\\
        &=\|(\theta\mathbf{T}_{1,a}+(1-\theta)\mathbf{T}_{2,a})\mathbf{p}_a-(\theta\mathbf{T}_{1,b}+(1-\theta)\mathbf{T}_{2,b})\mathbf{p}_b\|_1\notag\\
        &=\|\theta(\mathbf{T}_{1,a}\mathbf{p}_a-\mathbf{T}_{1,b}\mathbf{p}_b)+(1-\theta)(\mathbf{T}_{2,a}\mathbf{p}_a-\mathbf{T}_{2,b}\mathbf{p}_b)\|_1\notag\\
        &\underbrace{\leq}_{\substack{\text{Triangle}\\\text{Inequality}}}\|\theta(\mathbf{T}_{1,a}\mathbf{p}_a-\mathbf{T}_{1,b}\mathbf{p}_b)\|_1+\|(1-\theta)(\mathbf{T}_{2,a}\mathbf{p}_a-\mathbf{T}_{2,b}\mathbf{p}_b)\|_1\notag\\
        &=\theta\|\mathbf{T}_{1,a}\mathbf{p}_a-\mathbf{T}_{1,b}\mathbf{p}_b\|_1+(1-\theta)\|\mathbf{T}_{2,a}\mathbf{p}_a-\mathbf{T}_{2,b}\mathbf{p}_b\|_1\notag\\
        &=\theta f (\mathbf{T}_{1,a},\mathbf{T}_{1,b}) + (1-\theta) f(\mathbf{T}_{2,a},\mathbf{T}_{2,b})\notag
    \end{align}
    
    \underline{Joint convexity of the $Augmentation_1$ term:}

    We will show that every element of $\tilde{g}(\mathbf{T}_a,\mathbf{T}_b)=\max(\tilde{f}(\mathbf{T}_a,\mathbf{T}_b)-\mathbf{\tilde{F}}, \mathbf{0}_{1,2B+8})$ is jointly convex in $\mathbf{T}_a$ and $\mathbf{T}_b$. It directly follows from this that the $Augmentation_1$ term is jointly convex in $\mathbf{T}_a$ and $\mathbf{T}_b$ since the inner product is a linear combination of jointly convex functions.
    \begin{align}
        &\tilde{g}(\theta\mathbf{T}_{1,a}+(1-\theta)\mathbf{T}_{2,a},\theta\mathbf{T}_{1,b}+(1-\theta)\mathbf{T}_{2,b})\notag\\
        &=\max(\tilde{f}(\theta\mathbf{T}_{1,a}+(1-\theta)\mathbf{T}_{2,a},\theta\mathbf{T}_{1,b}+(1-\theta)\mathbf{T}_{2,b})-\mathbf{\tilde{F}}, \mathbf{0}_{1,2B+8})\notag\\
        &=\max(\hat{\mathbf{s}}^{(F)}\big[\tilde{\mathbf{M}}\circ(\tilde{\mathbf{I}}_2(\theta\mathbf{T}_{1,a}+(1-\theta)\mathbf{T}_{2,a})\tilde{\mathbf{I}}_1^T)\big]\tilde{\mathbf{P}}_a \notag\\
        & \ \ \ \ \ \ \ \ \ \ \ \ \ \ \ -\hat{\mathbf{s}}^{(F)}\big[\tilde{\mathbf{M}}\circ(\tilde{\mathbf{I}}_2(\theta\mathbf{T}_{1,b}+(1-\theta)\mathbf{T}_{2,b})\tilde{\mathbf{I}}_1^T)\big]\tilde{\mathbf{P}}_b-\mathbf{\tilde{F}}, \mathbf{0}_{1,2B+8})\notag\\
        &=\max(\theta\{\hat{\mathbf{s}}^{(F)}\big[\tilde{\mathbf{M}}\circ(\tilde{\mathbf{I}}_2\mathbf{T}_{1,a}\tilde{\mathbf{I}}_1^T)\tilde{\mathbf{P}}_a - \tilde{\mathbf{M}}\circ(\tilde{\mathbf{I}}_2\theta\mathbf{T}_{1,b}\tilde{\mathbf{I}}_1^T)\big]\tilde{\mathbf{P}}_b-\mathbf{\tilde{F}}\}\notag\\
        & \ \ \ \ \ \ \ \ \ \ \ \ \ \ \ + (1-\theta)\{\hat{\mathbf{s}}^{(F)}\big[\tilde{\mathbf{M}}\circ(\tilde{\mathbf{I}}_2\mathbf{T}_{2,a}\tilde{\mathbf{I}}_1^T)\tilde{\mathbf{P}}_a - \tilde{\mathbf{M}}\circ(\tilde{\mathbf{I}}_2\theta\mathbf{T}_{2,b}\tilde{\mathbf{I}}_1^T)\big]\tilde{\mathbf{P}}_b-\mathbf{\tilde{F}}\}, \mathbf{0}_{1,2B+8})\notag\\
        &\underbrace{\leq}_{\text{\makebox[0pt]{Elementwise}}}\theta\max(\hat{\mathbf{s}}^{(F)}\big[\tilde{\mathbf{M}}\circ(\tilde{\mathbf{I}}_2\mathbf{T}_{1,a}\tilde{\mathbf{I}}_1^T)\tilde{\mathbf{P}}_a - \tilde{\mathbf{M}}\circ(\tilde{\mathbf{I}}_2\theta\mathbf{T}_{1,b}\tilde{\mathbf{I}}_1^T)\big]\tilde{\mathbf{P}}_b-\mathbf{\tilde{F}},\mathbf{0}_{1,2B+8})\notag\\
        & \ \ \ \ \ \ \ \ \ \ \ \ \ \ \ +(1-\theta)\max(\hat{\mathbf{s}}^{(F)}\big[\tilde{\mathbf{M}}\circ(\tilde{\mathbf{I}}_2\mathbf{T}_{2,a}\tilde{\mathbf{I}}_1^T)\tilde{\mathbf{P}}_a - \tilde{\mathbf{M}}\circ(\tilde{\mathbf{I}}_2\theta\mathbf{T}_{2,b}\tilde{\mathbf{I}}_1^T)\big]\tilde{\mathbf{P}}_b-\mathbf{\tilde{F}}, \mathbf{0}_{1,2B+8})\notag\\
        &=\theta\max( \tilde{f}(\mathbf{T}_{1,a},\mathbf{T}_{1,b}),\mathbf{0}_{1,2B+8}) +(1-\theta)\max(\tilde{f}(\mathbf{T}_{2,a},\mathbf{T}_{2,b}), \mathbf{0}_{1,2B+8})\notag\\
        &=\theta \tilde{g}(\mathbf{T}_{1,a},\mathbf{T}_{1,b}) +(1-\theta)\tilde{g}(\mathbf{T}_{2,a},\mathbf{T}_{2,b}) \notag
    \end{align}
    
    \underline{Joint convexity of the $Augmentation_2$ term:}
    
    The convexity of this term follows from the fact that $\tilde{g}(\mathbf{T}_a,\mathbf{T}_b)$ is a non-negative jointly convex function in $\mathbf{T}_a$ and $\mathbf{T}_b$ and the norm is convex and non-decreasing over the set $\mathbb{R}_{\geq 0}$.
\end{proof}

\section{Experimental Details}
\label{exp_details}
Our experiments are conducted on three benchmark tabular datasets with known biases in the sensitive attributes; namely, the Adult \cite{kohavi1996scaling}, Law \cite{wightman1998lsac}, and Dutch Census \cite{van20002001} datasets. We treat race as the sensitive attribute for the Law dataset and sex as the sensitive attribute in the Adult and Dutch Census datasets. All codes were written in Python using version 3.8. In the remainder of this section, we provide the experimental details used in each of the three core modules of our analysis.

\subsection{Generator and Vector Quantization Implementation} 
To model the population distributions of each dataset, we use Conditional Tabular Generative Adversarial Network (CT-GAN) \cite{xu2019modeling}, which is the state-of-the-art for generating mixed continuous and discrete tabular data. For each dataset we train a generator for 400 epochs using the publically available code provided by ~\cite{xu2019modeling} \footnote{https://github.com/sdv-dev/CTGAN}. Each generator is then used to produce one million samples for each dataset, on which vector quantization is applied using the Linde, Buzo, Gray (LBG) algorithm, which is a multi-dimensional generalization of the Lloyd-Max algorithm for vector quantization. To produce a cell decomposition of size $N$, this iterative algorithm works by specifying an initial set of $N$ centroids. In each iteration, each centroid is updated by taking the average of all samples that are closer to it than any other centroid. This iterative process continues until convergence, specified by a relative error tolerance. We use a publicly available implementation of this algorithm \footnote{https://github.com/internaut/py-lbg}, specifying a relative error of 0.01 as the stopping criterion for optimizing the VQ cell partitioning. We set  $N=256$ for analyzing the population distribution of each dataset in the body of the main paper.

\subsection{Solving Minimization Problem (4) and (30)} 
The linear programs used to perform the fairness-accuracy trade-off analysis in minimization problems (4) and (30) were implemented using the linprog tool from the Scipy Optimize library. For each of the datasets, we normalize the feature vectors for performing the distance calculations associated with local individual fairness. Since each of the analyzed datasets are composed of a mixture of discete and continuous features, we use the Hamming distance to calculate the elementwise distances between discrete entries of a feature vector and the absolute distance to measure the distances between continuous features. All continuous features are normalized to have zero mean and $\frac{1}{2}$ variance so that the maximum distance of each entry is approximately $1$ \cite{wilson1997improved}. This was done to ensure that each element of a feature vector has approximately equal contribution to the final distance. The average of these distances is taken as the final distance. For all experiments involving local individual fairness constraints, the following procedure was used for choosing $\eta.$ The distances between each pair of cell centroids in our discretized population distribution were calculated. Of this set of distances, the $n^{\text{th}}$--percentile was calculated. All pairs of distances smaller than this percentile represent local neighboring cell centroids. Empirically, we set $n=3.5$ percentiles to not include too few or too many neighbors in a cell's local neighborhood. Furthermore, we set the parameter  $\theta=1$ in the exponential of the $(Ind)$ constraints for both minimization problems.

\subsection{Solving Minimization Problem (9) and (32)} 
The implementation of the method of multipliers used for solving minimization problem (9) was performed using Tensorflow, version 2.8. To perform the updates specified by equation (10) to $\mathbf{T}$ in each iteration, we perform gradient descent with a momentum of 0.9 and a decaying learning rate from 1e-2 to 1e-12. The entire minimization process is terminated once the sum of square residuals for $\bm{\rho}$ and $\mathbf{T}$ falls below 1e-4.

A similar process is used to solve minimization problem (32) with the following modifications. In each iteration we sequentially minimize $\mathbf{T}_a$ and $\mathbf{T}_b$ (using the same parameters specified above). Furthermore, the entire minimization process is halted once the sum of square residuals for $\bm{\rho}$, $\mathbf{T}_a$, and  $\mathbf{T}_b$ falls below 1e-4.

\section{Full Set of Pareto Frontiers for Sensitive-Unaware}
\label{extra_unawareness_pareto}
Since spacing limitations prevented us from including more plots in Fig. \ref{Acc_fair_trade} of the main paper, we provide the remaining Pareto frontiers under unawareness of the sensitive attribute that could not be fit into that figure in Fig. \ref{remaining_pareto_fig} below. As was the case for the sensitive-attribute-aware plots in Fig. \ref{Acc_fair_trade}, for the Adult and Law datasets, we see much smaller accuracy dropoffs in pairings involving DP as opposed to pairings involving EA. The converse is true for the Dutch Census dataset, further highlighting the distributional dependence of the tensions that exist among these fairness notions.

\begin{figure}
\centerline{\includegraphics[width=\columnwidth]{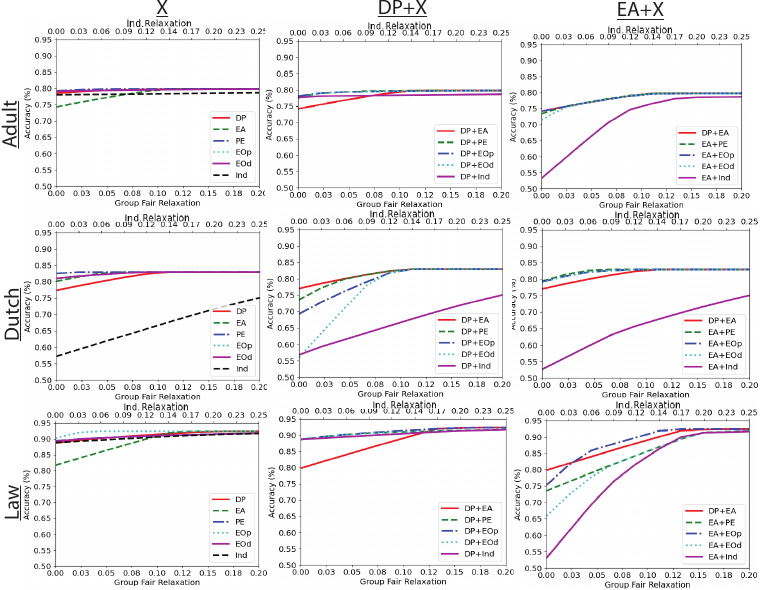}}
\caption{Pareto frontiers capturing the accuracy-fairness trade-off for three datesets under unawareness of the sensitive attribute.  Each plot provides curves for different pairings of fairness constraints; namely, DP, EA, PE, EOd, and Ind.} 
\label{remaining_pareto_fig}
\end{figure}

\begin{figure}
\centerline{\includegraphics[width=0.99\columnwidth]{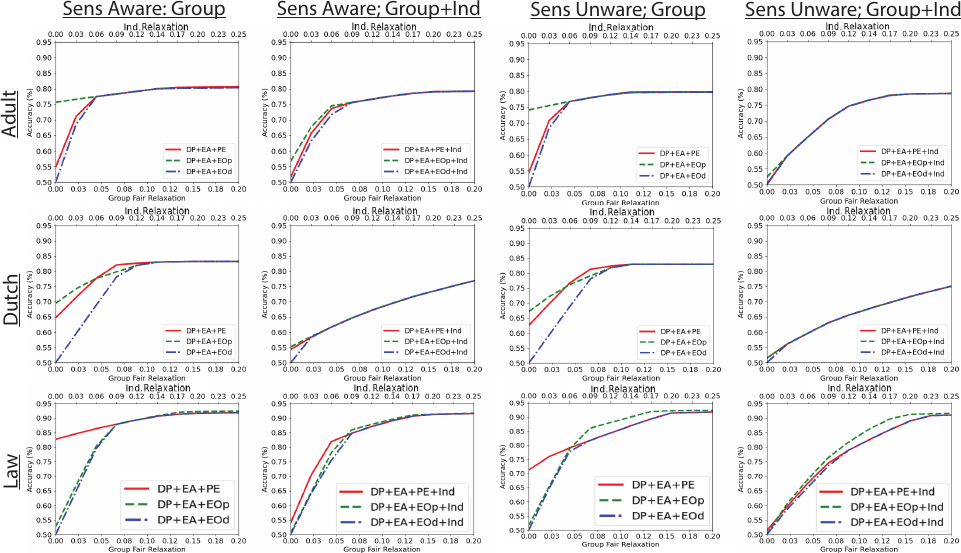}}
\caption{Pareto frontiers capturing the accuracy-fairness trade-off for three datesets under awareness (columns 1 and 2) and unawareness (columns 3 and 4) of the sensitive attribute. Plots include pairs of three group fairness definitions with (columns 2 and 4) and without (columns 1 and 3) individual fairness constraints.}
\label{three_pareto_fig}
\end{figure}

All Pareto frontiers discussed up to this point deal with no more than two combinations of fairness notions. In Fig.~\ref{three_pareto_fig} we provide plots containing three or more combinations of fairness definitions. It can clearly be seen that the accuracy droppoffs in these plots are much more drastic than those from Fig.~\ref{Acc_fair_trade} or ~\ref{remaining_pareto_fig}, as is to be expected. Noteably, enforcing local individual fairness on top of any combination of three group fairness prohibits any meaningful classification result with all accuracies dropping to around 50\% with the inclusion of the Ind constraint inside this figure.

\section{VQ Granularity vs Time Complexity}
\label{VQ_time}

Given that the generator has produced a sufficient number of samples for modeling the population distribution, fine-grain partitioning the space of feature vectors is desirable for precise analysis since increasing the number of VQ cells reduces their size, leading the average information associated with a cell centroid to be more representative of all the feature vectors within a cell. However, there is a trade-off between the number of VQ cells and optimization time complexity.  We selected 256 as the number of VQ cells for our analysis to strike a balance between precision and time complexity. In Table \ref{VQ_size_tab}, we report the average accuracy and standard deviation of the Pareto frontiers for different combinations of group fairness constraints for relaxations running from 0 to 0.2. The results show that significantly increasing the number of VQ cells from 256 to 512 cells only leads to marginal changes in the accuracy and standard deviation of the Pareto frontiers for each dataset, indicating that the trends have begun to plateau. These results remain consistent across datasets. Hence, 256 cells provides a faithful representation of the fairness-accuracy trade-off.

In Table \ref{time_complexity_tab}, we report the time complexities associated with the different modules of our analysis. All experiments were run on a Macbook Pro (1.7 GHz Quad-Core Intel Core i7) with no GPU support. Training the codebook for the 256 VQ cell decomposition is the most expensive of all the tasks that we preformed, but the time complexities are reasonable, especially since training only needs to be performed once for each dataset. The results for the four remaining tasks in the table consist of averages and standard deviations of the computational complexities for solving each minimization problem in our optimization framework for all relaxations of each group and individual fairness combination reported in Figures 2 and 5 and Tables 2-4. The linear programming minimization problems are able to be solved extremely efficiently with runtimes all far below one second for processing.  Optimization problems (9) and (32) are much more expensive since a large number of gradient steps must be applied for the solution to converge. The runtimes for these problems tend to be much smaller when only pairs of group fairness constraints are active since each group fairness definition is associated with just one constraint (in constrast to individual fairness). Still, the average runtimes for solving problems (9) and (32) are still typically around 30 minutes, with worst case runtimes typically no worse than one hour and 30 minutes. Using more adaptive learning rate schedules could further improve these runtimes.

\begin{table}[H]
\caption{Average accuracy of group fairness Pareto frontiers with relaxations varied between 0 and 0.2. Standard deviations of the accuracies of each frontier are provided in parentheses.}
\label{VQ_size_tab}
\resizebox{\textwidth}{!}{
\begin{tabular}{ c  c c  c c c c c}
  \hline
  \multicolumn{8}{c}{Adult} \\
 \hline
  & \multicolumn{3}{c}{Awareness} &  & \multicolumn{3}{c}{Unawareness} \\
 \hline
 \# Cells & DP+EA & DP+EOd & EA+EOd & & DP+EA & DP+EOd & EA+EOd\\
 \hline
 16    &  0.734  (0.039) & 0.775  (0.000) &  0.720  (0.059) &  & 0.701    (0.062) & 0.775    (0.000) & 0.699    (0.065)\\
 32   &  0.763  (0.022) &  0.784 (0.004) &  0.750 (0.047) &   &   0.754 (0.029) & 0.781 (0.002) &  0.744 (0.047)\\
 64  &  0.776 (0.019) &  0.786 (0.005) &  0.768 (0.028) &   &  0.763 (0.026) &  0.782 (0.002) &  0.760 (0.032)\\
 128 &  0.778  (0.019) &  0.789 (0.005) &  0.772 (0.026) &   &   0.769 (0.025) &  0.786 (0.004) & 0.766 (0.030)\\
 256  &  0.789 (0.017) &  0.798 (0.008) &  0.785 (0.023) &   & 0.782 (0.020) &  0.793 (0.007) &  0.779 (0.026)\\
 512 &  0.803  (0.015) &  0.808 (0.012) &  0.798 (0.023) &   &   0.798 (0.018)  &  0.805  (0.011) &  0.794 (0.026)\\
 \hline
  \multicolumn{8}{c}{Dutch Census} \\
 \hline
  & \multicolumn{3}{c}{Awareness} &  &  \multicolumn{3}{c}{Unawareness} \\
 \hline
 \# Cells & DP+EA & DP+EOd & EA+EOd & & DP+EA & DP+EOd & EA+EOd\\
 \hline
 16    &  0.688 (0.008) &  0.672 (0.042) & 0.697 (0.007) &   &  0.674  (0.001) &  0.657 (0.035) & 0.667 (0.012)\\
 32   &  0.728  (0.011) &  0.704 (0.057) & 0.735 (0.011) & & 0.719 (0.007) &  0.695 (0.052) & 0.717 (0.010)\\
 64  &  0.791 (0.014) &  0.747 (0.082) & 0.797 (0.012) &   &  0.783 (0.014) & 0.741 (0.078) &  0.784(0.011)\\
 128 &  0.796  (0.015) &  0.752 (0.084) & 0.801 (0.015) &   & 0.790 (0.014) & 0.746 (0.081) & 0.787 (0.016)\\
 256  &  0.821 (0.013) &  0.766 (0.093) &  0.826 (0.010) &   &  0.814  (0.020) &  0.765 (0.093) & 0.822 (0.012)\\
 512 &  0.839  (0.014) &  0.778 (0.101) & 0.848 (0.007) &   &  0.829  (0.024) & 0.777 (0.100) &  0.846 (0.008)\\
 \hline
  \multicolumn{8}{c}{Law} \\
 \hline
  & \multicolumn{3}{c}{Awareness} &  &  \multicolumn{3}{c}{Unawareness} \\
 \hline
 \# Cells & DP+EA & DP+EOd & EA+EOd & & DP+EA & DP+EOd & EA+EOd\\
 \hline
 16    &  0.823  (0.049) &  0.891 (0.002) & 0.821 (0.052)  &    & 0.803 (0.063) &  0.890  (0.002) &  0.765  (0.100)\\
 32   &  0.841  (0.046) &  0.894 (0.004) & 0.839 (0.048) &    &  0.821  (0.060) &  0.891 (0.003) &  0.791  (0.085)\\
 64  &   0.844 (0.045) &  0.896 (0.004) & 0.845 (0.046) &   &  0.823 (0.059) & 0.894(0.004) &  0.796  (0.070)\\
 128 &  0.849 (0.045) &  0.899 (0.006) &  0.853 (0.045) &   &  0.832  (0.057) &  0.897  (0.005) &  0.819  (0.072)\\
 256  & 0.891 (0.032) & 0.910  (0.010) &  0.876 (0.055) &   &  0.879 (0.045) &  0.907 (0.009) &  0.833 (0.085)\\
 512 &  0.902  (0.031) &  0.916 (0.013) &  0.892  (0.050) &   &   0.888 (0.044) & 0.914 (0.012) &  0.855 (0.077)\\
 \hline
\end{tabular}}
\end{table}




\begin{table}[H]
\centering
\caption{Average time complexity of the different optimization tasks in the proposed framework for each dataset.}
\label{time_complexity_tab}
\begin{tabular}{ c c c c }
 \hline
  Task& Adult & Dutch Census & Law\\
 \hline
 256 VQ Cell Training    & \ \ 5375s \ \ \ \ (-----) \ & \ \ 8513s \ \ \ \ (-----) \ & \ \ 8417s \ \ \ \ (-----) \\
 Solving (4)   & \ 0.007s \ (0.007) \ & \ 0.007s \ (0.007) \ &  \ 0.007s \ (0.008)  \\
 Solving (9)  & \ \ \ \ 839s \ \ \ \ (525) \ & \ \ 1348s \ \ \ \ (941) \ & \ \ 1644s \ \ (1823)  \\
 Solving (30)  & \ 0.033s \ (0.009) \ & \ 0.043s \ (0.008) \ & \ 0.035s \ (0.009) \\
 Solving (32) & \ \ 2773s  \ \ (2327) \ & \ \ 2104s \ \ (1855) \ & \ \ 1785s \ \ (1743)  \\
 \hline
\end{tabular}
\end{table}

\newpage
\section{Summary of Notation}
\label{notation_summary}
\begin{table}[h]\caption{Notation}
\begin{center}
\begin{tabular}{r c p{10cm} }
\toprule
bold face capital letter (e.g. $\mathbf{X}$) & $\triangleq$ & Matrix\\
bold face lowercase letter (e.g. $\mathbf{x}$) & $\triangleq$ & Column vector\\
$\mathbf{X}[i,j]$& $\triangleq$ & Element in $i^{th}$ row and $j^{th}$ column of $\mathbf{X}$\\
$\mathbf{x}[i]$& $\triangleq$ & $i^{th}$ element of $\mathbf{x}$\\
$\mathbf{T}$ & $\triangleq$ & Matrix representing decorrelation mapping for feature vectors\\
$\mathbf{T}_a$ & $\triangleq$ & Matrix representing decorrelation mapping for feature vectors specific to Group $a$\\
$\mathbf{T}_b$ & $\triangleq$ & Matrix representing decorrelation mapping for feature vectors specific to Group $b$\\
D &$\triangleq$ & Dataset\\
G &$\triangleq$ & Sample generator\\
$A$ & $\triangleq$ & Sensitive attribute random variable\\
$Y$ & $\triangleq$ & Class label random variable\\
$X$ & $\triangleq$ & Feature vector random variable\\
$\mathcal{X},\mathcal{A},\mathcal{Y}$ & $\triangleq$ & Sample spaces for the feature vector, sensitive attribute, and class label random vector/variables\\
$S$ & $\triangleq$ & Randomized Scoring function \\
$\hat{Y}$ & $\triangleq$ & Class label estimator\\
$\mathbf{s}^{*}$ & $\triangleq$ & Scores produced by unconstrained Bayesian classifiers \\
$\mathbf{s}^{(F)}$ & $\triangleq$ & Scores produced by fair Bayesian classifiers\\
$\mathbf{m}$ & $\triangleq$ & deviation between scores produced by unconstrained and fair Bayesian classifiers (i.e. $\mathbf{s}^{*}$ and $\mathbf{s}^{(F)}$). We optimize for this vector.\\
$\mathbf{1}_{N}$ & $\triangleq$ & column vectors of length $N$ containing all $1$s \\
$\mathbf{0}_{N}$ & $\triangleq$ & column vectors of length $N$ containing all $0$s\\
$\mathbf{I}_{M,M}$ & $\triangleq$ & $M\times M$ identity matrix\\ $\mathbf{0}_{M,N}$ & $\triangleq$ & $M$ and $N$ matrix of all zeros\\
$\mathbf{1}_{M,N}$ & $\triangleq$ & $M$ and $N$ matrix of all ones\\
$\mathbf{p}$ & $\triangleq$ & Vector capturing distribution involving $X$ in which $X$ \textit{is not} a variable on which we condition. E.g. The $i^{th}$ element of $\mathbf{p}_0^a$ is equal to $P(X=\mathbf{x}_i,A=a|Y=0)$\\
\toprule
\end{tabular}
\end{center}
\label{tab:TableOfNotationForMyResearch}
\end{table}

\end{document}